\documentclass{article} 

\usepackage{natbib}

\usepackage{hyperref}
\usepackage{fullpage}
\usepackage{url}

\usepackage{wrapfig}

\usepackage{algorithm}
\usepackage{algorithmic}

\usepackage{amsmath,amsfonts,amssymb,amsthm}
\usepackage{caption}
\usepackage{subcaption}
\newcommand{\ssymbol}[1]{^{\@fnsymbol{#1}}}
\usepackage{amsthm}
\usepackage[inline,shortlabels]{enumitem}
\usepackage{cleveref}
\usepackage{thmtools,thm-restate}
\usepackage{booktabs}
\usepackage{graphicx}
\theoremstyle{definition}
\newtheorem{definition}{Definition}[section]
\newtheorem{claim}[definition]{Claim}
\newtheorem{fact}[definition]{Fact}

\usepackage{authblk}

\usepackage{amsmath,amsfonts,bm}

\def\eqref#1{equation~\ref{#1}}

\def\1{\bm{1}}

\DeclareMathAlphabet{\mathsfit}{\encodingdefault}{\sfdefault}{m}{sl}
\SetMathAlphabet{\mathsfit}{bold}{\encodingdefault}{\sfdefault}{bx}{n}

\DeclareMathOperator*{\argmax}{arg\,max}

\title{An Axiomatic Approach to Model-Agnostic Concept Explanations}

\author[1]{Zhili Feng\footnote{Contributed equally.}}
\newcommand\CoAuthorMark{\footnotemark[\arabic{footnote}]} 
\author[2]{Michal Moshkovitz\protect\CoAuthorMark}
\author[2]{Dotan Di Castro}
\author[1,2]{J. Zico Kolter}
\affil[1]{Carnegie Mellon University}
\affil[2]{Bosch Center for AI}

\begin{document}

\maketitle

\begin{abstract}
Concept explanation is a popular approach for examining how human-interpretable concepts impact the predictions of a model. However, most existing methods for concept explanations are tailored to specific models. To address this issue, this paper focuses on model-agnostic measures. Specifically, we propose an approach to concept explanations that satisfy three natural axioms: linearity, recursivity, and similarity. 
We then establish connections with previous concept explanation methods, offering insight into their varying semantic meanings.
%
Experimentally, we demonstrate the utility of the new method by applying it in different scenarios: for model selection, optimizer selection, and model improvement using a kind of prompt editing for zero-shot vision language models.

\end{abstract}

\section{Introduction}

In today's world, with the expanding influence of machine learning systems on our lives, it is crucial to understand the reasoning behind their decisions. Unfortunately, the current black-box nature of ML models often leaves users puzzled.
In response to this challenge, the field of explainable AI (XAI) has emerged, which aims to develop new types of explanations and methods to find them. One type of explanation that stands out is \emph{concept explanation}, which assesses the importance of a human-understandable concept to a prediction. For instance, in predicting the appearance of a ``zebra'' in an image, the presence of ``stripes'' is a human-understandable concept that may impact the prediction.

Several methods \citep{kim2018interpretability,bai2022concept} have been developed to capture the importance of a concept in relation to a prediction. However, these methods are model-specific and require white-box access, which can be a limitation, for example, in cases where the model is proprietary and cannot be freely accessed for inspection. Moreover, there is often a lack of clarity surrounding the interpretation of these methods for concept explanation, the specific aspects they evaluate, and how they differ from each other.

To address these challenges, we aim to develop an axiomatic approach for determining the importance of a concept in relation to a prediction. We begin by establishing a few axioms, derive measures that are consistent with these axioms, and design an algorithm for estimating the measures. 
Additionally, we examine the connection between our method and previous research in the field. 

The newly proposed framework offers a number of advantages:
\begin{enumerate}
    \item \textbf{Axiomatically-based measures.} 
    Our proposed measures are rooted in three axioms, with two of them having been implicitly assumed in prior studies. This axiomatic approach enables us to develop measures that are firmly established.
     \item \textbf{Model-agnostic.} 
    Our measures are model-agnostic, i.e., no need to know the internal workings of the model to apply the new measures. 
     \item \textbf{Integration of previous research.} We establish a connection between prior work (TCAV \citep{kim2018interpretability} and  completeness-aware \citep{yeh2020completeness} explanations), and our approach. 
     
    \item \textbf{Semantic meaning of concept-explanation methods.} The new framework helps understand the semantic meaning of prior concept explanation methods: TCAV \citep{kim2018interpretability} is associated with necessity, while completeness-aware \citep{yeh2020completeness} with sufficiency. 
    
     \item \textbf{Efficient Implementation.} We design an efficient implementation for estimating the new measures, which is more efficient than previous methods for concept-explanations.
\end{enumerate}

We use the new method we developed for a variety of applications, see Section~\ref{sec:exp_applications} for more detaild: 
\begin{enumerate}
    \item \textbf{Model selection.} 
    Our approach is model-agnostic, allowing it to explain various models and thus facilitate model selection. We use it to make a comparison between logistic regression and random forest, assisting users in understanding different aspects of the learned models in order to enable them to make an informed decision as to which model to select. 
    \item \textbf{Optimizer selection.}
We can use the new method to select between different optimizers. Namely, we apply our method on networks trained using different optimization algorithms: SGD and AdamW. Our method offers insights into the differences between these two optimization techniques, thereby assisting in the selection process.
    \item \textbf{Model improvement via prompt editing.} Through our approach, we are able to improve model performance by mitigating the influence of irrelevant concepts. Specifically, we use our method to detect concepts impacting the prediction but lack relevance to the classification task. We then identify these concepts in the CLIP embedding space. We demonstrate that by editing the CLIP textual embedding of each class and decreasing the representation of the irrelevant concepts, we can significantly enhance model performance.
\end{enumerate}

 Our method relies on having examples annotated with relevant concepts. However, when these annotations are unavailable, we show, in Section~\ref{sec:automatic_concept_labeling}, how to leverage recent developments in image captioning to automatically label concepts in a semantically meaningful manner.

\section{Related Work}\label{sec:related}

\paragraph{Post hoc explanations.}
Post hoc explanations can offer valuable insights into how a model works without compromising the model's accuracy. These explanations provide additional information, such as feature-attributions \citep{lundberg2017unified,ribeiro2016should,ribeiro2018anchors} that evaluate the contribution of each feature to the prediction. In addition to feature-attribution explanations, there are also example-based approaches like counterfactual explanations and influential examples. Counterfactual explanations \citep{wachter2017counterfactual,karimi2020model} highlight what changes to input would have to be made to change the model's output. Influential examples \citep{koh2017understanding,koh2019accuracy} reveal which examples are the most important for a model's decision-making process. Overall, these post hoc explanations can help provide a better understanding of a model's inner workings. Concept explanations, the focus of this paper, are a form of post hoc explanations that focus on determining the impact of concepts on a model's predictions.

\paragraph{Concept-based explanations.}
Different studies have proposed methods to return the influence of concepts for neural network models \citep{kim2018interpretability,bai2022concept,yeh2020completeness}, unlike our approach these methods are not model-agnostic.
 Another line of research \citep{ghorbani2019towards,yeh2020completeness} has sought to learn relevant concepts for a given task. 
This paper proposes a novel axiomatic approach to post hoc concept explanations, which is model-agnostic and thus applicable to a wide range of models. In contrast to post hoc explanations, \citet{koh2020concept,espinosa2022concept, yuksekgonul2022post,chen2019looks, bontempelli2022concept} propose a different approach to understanding concepts, using concept bottleneck models, where concepts are part of the model's hidden representation.

\textbf{Model-agnostic explanations.} 
    Model-agnostic explanations are methods that produce explanations that do not require any knowledge about the internal workings of the model. The lack of access to the inner workings of a model can be particularly beneficial when the model is confidential and cannot be revealed. 
    Numerous model-agnostic techniques are available for feature attribution (e.g., \cite{lundberg2017unified,ribeiro2016should}). However, when it comes to concept explanations, there are fewer model-agnostic approaches. Our approach is model-agnostic, distinguishing it from other techniques, as discussed in
    \citet{kim2018interpretability,bai2022concept} which require access to the model's gradients or hidden layers.

\paragraph{Axiomatic approach to feature attribution explanations.} 
Feature attribution explanation assigns a value to each feature to represent its impact on the prediction. This explanation is a type of concept explanation where the concepts are restricted to the feature values (for neural networks, the features are usually the internal representations of the inputs). To ensure a consistent and reliable explanation method, many researchers have suggested using an axiomatic approach to this type of explanation \citep{lundberg2017unified,sundararajan2017axiomatic}. For example, the implementation invariance axiom proposed by \citet{sundararajan2017axiomatic} is model-agnostic and requires identical attributions for two functionally equivalent models. 

\paragraph{Understanding the explanations.} Concept explanations generalize the notion of feature attribution, which solely assigns significance to features in the prediction. However, even within this more narrow explanation category, 
    different attribution methods \citep{lundberg2017unified,ribeiro2016should} produce varying and sometimes inconsistent results \cite{krishna2022disagreement}. While one approach to understand these differences is to show they have a similar mathematical structure \citep{han2022explanation}, a different approach is suggested in this paper. Specifically, we propose to understand the semantic differences when comparing different concept-explanations methods. We demonstrate that TCAV \citep{kim2018interpretability} is associated with the necessity of the concept to the prediction, while completeness-aware \citep{yeh2020completeness} is associated with the sufficiency.


\section{The Axiomatic Approach}\label{sec:axiomatic_approach}
\paragraph{Notation.} 
We use $\mathcal{X}$ to represent the set of input examples. The model to be explained is denoted by $h: \mathcal{X} \rightarrow \{-1,+1\}$, which takes inputs from $\mathcal{X}$ and outputs either $+1$ or $-1$ to indicate whether the prediction is true or false, respectively. The concept we are interested in measuring is denoted by  $c:\mathcal{X} \rightarrow [-1,+1]$ (or $c:\mathcal{X} \rightarrow \{-1,+1\}$ in the discrete case), which maps inputs from $\mathcal{X}$ to values between $-1$ and $+1$, where $c(x) = 1$ indicates that the concept holds and $c(x)=-1$ indicates that it does not. Note that a concept is an intrinsic property of an example and is entirely unrelated  to any prediction made by a model. While for simplicity we introduce the framework with the binary classification problem, it is straightforward to extend to the $k$-class case by treating it as $k$ one-vs-all classification problems.
The probability distribution over examples is denoted by $p:\mathcal{X}\rightarrow[0,1]$, which maps inputs from $\mathcal{X}$ to a probability, $\sum_x p(x)=1$. The proofs are deferred to the appendix.
\paragraph{Problem formulation.}
Our goal is to identify a measure, denoted as $M$, which quantifies the influence of a concept $c$ on the predictions of an opaque model $h$; for example, consider $h$ as a predictor of the ``zebra'' class and $c$ is the concept of ``stripes''. This allows one to gain a deeper understanding of the model $h$ by making use of a concept that is easily understandable to humans. Towards achieving this goal, the first step is to examine the inputs of $M$. Specifically, the measure $M$ is a function maps the following parameters to $\mathbb{R}$: (i) the model $h$, (ii) the concept  $c$, and (iii) the a priori probability $p(\cdot)$ over the examples. 


\subsection{Axiom 1: Linearity with respect to examples} 
The first axiom states that when measuring the impact of a concept, the impact on each example, which is determined solely by the example's properties, should be summed.
For instance, when evaluating the impact of the concept ``stripes'' on predicting a ``zebra'', one should separately assess how impactful ``stripes'' are for each example $x$ and then sum everything together. The knowledge of how much stripes are impacting another example $x'$ does not influence how stripes are impacting example $x$.
Mathematically, this axiom implies that $M$ can be expressed as
\[
M(h,c,p)=\sum_{x\in\mathcal{X}} m(h(x), c(x), p(x)), 
\]
for some local impact function $m:\{-1,+1\}\times[-1, 1]\times[0,1]\to\mathbb{R}$.

The linearity axiom is presumed in numerous other situations, such as the standard evaluation of a hypothesis using an \emph{additive} loss function (e.g., Chapter $2$ in \citep{shalev2014understanding}). Additionally, previous techniques for concept explanations, such as TCAV, also made this assumption. For further information and comparison between various measures for concept explanations and our work, see \Cref{sec:comparison}.

\subsection{Axiom 2: Recursivity}
The second axiom focuses on the local measure $m$ and investigates how it relates the probability example, $p(x)$, to the properties of example $h(x)$ and $c(x)$.  The axiom specifically aims to investigate how $m$ changes with respect to $p(x)$.  This axiom draws parallels to the entropy recursivity axiom mentioned in \citep{csiszar2008axiomatic}.

The recursivity axiom requires that when splitting the probability $p(x)$ of an example between two new examples with the same $h(x)$ and $c(x)$ values, the value of $M$ remains unchanged. In other words, the axiom requires that replacing an example $x$ with two examples that have probabilities $p_1$ and $p_2$ such that $p_1+p_2 = p(x)$ and the same $h(x)$ and $c(x)$ values, keeping all other examples unchanged, does not alter the value of $M$. We can write $M$ in the following way. 


\begin{restatable}{claim}{AxiomTwoClm}
\label{clm:axiom2}
If $M$ satisfies the first two axioms and the range of $p$ is the rationals, then $M$ can be written  as 
\[ 
M(h, c, p) = \sum_x p(x) \cdot s(h(x), c(x)),
\]
for some function $s:\{-1,+1\}\times [-1,+1]\to\mathbb{R}.$
\end{restatable}

This axiom guarantees linearity in probability. Similar to the earlier axiom, prior methods for assessing concept explanations, including TCAV, have also implicitly relied on this assumption.

\subsection{Axiom 3: Similarity}
In the third axiom, we examine the term $s(h(x),c(x))$, which focuses on a single example with a certain probability (the exact probability is irrelevant for $s$). It aims to determine the contribution of $h(x)$ and $c(x)$ to the overall value of $M$. The similarity axiom asserts that for every example $x$, when $c(x)$ and $h(x)$ are close, $M$ will be high, and when they are far, $M$ will be low. To satisfy this axiom, any similarity function can be utilized.
One natural choice is to take $s(h(x), c(x)) = c(x)\cdot h(x)$, which treat the class and concept similarly. 
As an illustration, consider evaluating the influence of "stripes" on a "zebra" predictor. When stripes are prominently featured in an instance (i.e., a high value of $c(x)$), it can influence the predictor only if a zebra is also present (i.e., $h(x)=1$).

\subsection{The new measures and estimation algorithm}
We can explore different instantiations of $s$, either focus on a specific class or concept, or treat them symmetrically, resulting in the following measures.   
\begin{definition} Fix predictor $h: \mathcal{X} \rightarrow \{-1,+1\}$, concept  $c:\mathcal{X} \rightarrow [-1,+1]$, and probability $p$ over the examples, 
    \begin{itemize}
    \item (symmetric measure) For $s(h(x), c(x)) = c(x)\cdot h(x)$, \[M(h, c, p) = \mathbb{E}_{x\sim p}[h(x) c(x)].\]
    \item (class-conditioned measure) For  $s(h(x), c(x)) = c(x)$ if $h(x)=1$ and otherwise undefined, after normalization  \[
M(h, c, p) = \mathbb{E}_{x\sim p}[c(x)| h(x) = 1].
\]
     \item (concept-conditioned measure) For $s(h(x), c(x)) = h(x)$ if $c(x)$ is high (i.e., $c(x)\geq\theta$ in the continuous case and $c(x)=1$ in the discrete case) and otherwise undefined, after normalization 
     \[
M_\theta(h, c, p) = \mathbb{E}_{x\sim p}[h(x)| c(x) \geq \theta].
\]
\end{itemize}
\end{definition}

Intuitively, class-conditioned measure quantifies how often the concept $c$ occurs in the class, e.g., if $c$ is always present when $h(x)=1$, then the value of $M$ will be high; and concept-conditioned measure quantifies the probability of the class when the concept holds (i.e., exceeding the threshold), e.g., if whenever $c$ is present enough (i.e., $c(x)\geq\theta$), the class is 1, then $M$ will be high. Putting differently, 
\begin{center}
\textit{class-conditioned measures \emph{necessity} of $c$ to $h$ and concept-conditioned the \emph{sufficiency}}
\end{center}


We can estimate the symmetric measure using the following simple algorithm (and in a similar manner for all the three measures).  
\begin{algorithm}[H]
\caption{Algorithm for estimating the symmetric measure}
\begin{algorithmic}[1] 
\STATE \textbf{Input:}  $n$ examples $x_1, x_2, \ldots, x_n$ sampled from distribution $p$ and labeled by $h$ and $c$ 
\RETURN $\frac{1}{n}\sum_{i=1}^n h(x_i)c(x_i)$
\end{algorithmic}
\end{algorithm}

\section{TCAV and Completeness-aware Explanations within the Axiomatic Framework}\label{sec:comparison}

This section provides formal proof that previous research conducted on concept explanations can be viewed as a part of our new axiomatic approach, given certain assumptions. Specifically, we focus on two widely used methods: TCAV \citep{kim2018interpretability} and Completeness-Aware \citep{yeh2020completeness} Explanations.
A benefit of our approach is that it enables a faster implementation of these previous methods, as compared to their original implementation which required learning a model. 
%

\paragraph{Setting.} 
In prior research, ground truth labeling $\{(x_i, y_i)\}_i$ was utilized to assess the quality of concept explanations. However, as our objective is to evaluate the influence of explanations on the model, it is necessary for the evaluation measure to be independent of ground truth. For instance, when examining the influence of a particular profession on a loan decision, the actual loan repayment outcome becomes insignificant in understanding the impact of the profession \emph{on the model}. Therefore, when comparing to previous research, we define $y_i=h(x_i)$ in order to eliminate the dependence on the ground truth.

\subsection{Completeness-aware Concept-based Explanation}\label{sub_sec:comapre_complete_aware}
\citet{yeh2020completeness} propose a completeness score to evaluate the quality of a set of concepts for a prediction task. The primary term, i.e., the only one that depends on the concept, in the completeness score is the following (see Appendix~\ref{apx:proof_comparison} for the full completeness score), 
\begin{equation}
S_{CACBE}\triangleq\sup_g \Pr_{x,y}[y = \argmax_{y'} h_{y'}(g(c(x)))],
\end{equation}
where $h_{y'}$\footnote{Note that we have overloaded the notation to allow the input of 'h' to be a representation of an example, rather than just the example itself.} is the predicted probability of label $y'$, and $g$ is a function that maps the set of concepts to the feature space.  In our notation, $\argmax_{y'}h_{y'}$ is simply $h$. 
The score evaluates the quality of the explanation by analyzing the degree to which it is sufficient for decoding $h(x)$ with the help of $g$.

We show that the completeness score can be written as a simple aggregation of the concept-conditioned measure for the different classes, when $\mathcal{C}(x)=c(x)$, and the concepts and predictions are binary. Namely, 
\begin{restatable}{theorem}{CompareAware}
\label{clm:compare_aware}
For concept $c:\mathcal{X}\rightarrow\{-1,+1\}$ and predictor $h:\mathcal{X}\rightarrow\{-1,+1\}$, the following holds
\begin{equation}\label{clm:aware_equiv}
    S_{CACBE}=\frac12+\frac12\sum_{\ell=1,-1}\rvert\mathbb{E}_{x}[h(x)|c(x)=\ell]\lvert\cdot \Pr(c(x)=\ell).
\end{equation}
\end{restatable}
In other words, \citet{yeh2020completeness} implicity applies Axiom 3 by treating each class independently and summing them all together to create the final similarity measure.

\paragraph{Efficient implementation.}
In order to compute the completeness score, one must first learn $g$, which can be a time-intensive process. In contrast, our approach eliminates the need for this learning phase. As a result, our measures can be computed more efficiently than those in the work of \citet{yeh2020completeness}.

\subsection{Testing with Concept Activation Vectors (TCAV)} \label{sub_sec:comapre_tcav}
The problem of evaluating the impact of concepts was first introduced in the work of \citet{kim2018interpretability}. The approach involves using (i) a set of examples that are labeled with the concept of interest, and (ii) the gradient of the model's prediction for each example. They assumed  the concept is linearly separable, by a unit vector $v$, in the representation domain of the examples. Mathematically, \citet{kim2018interpretability} quantifies the concept $c$ for example $x$ as $c(x) = g(x) \cdot v$, for some embedding function $g(\cdot)$.
To ensure that $c(x)\in[-1,+1]$, we make the assumption that unit norm embeddings are used. 
As a side-note, subsequent research \citep{bai2022concept} relaxed the linearity assumption and employed the concept's gradient, rather than utilizing $v$. 

The TCAV definition involves computing the ratio of examples in a specific class, $X_k$, 
that has a positive score. The score is computed as the dot product between the gradient of the model's prediction with respect to the representation $g(x)$ and the concept unit vector $v$. Quantitatively, the TCAV definition is 
\[
\textrm{TCAV} := \frac{\{|x\in X_k : S(x) > 0\}|}{|X_k|}, 
\]
where  $S(x) = \nabla f(g(x))\cdot v$ and $f$ is some function.
Our focus is on a continuous version of TCAV:
$
\textrm{TCAV$_{con}$} := \frac{\sum_{x\in X_k } S(x)}{|X_k|}. 
$
The work of \citet{kim2018interpretability} investigates the impact of a concept on various layers of a neural network by utilizing the representations in those layers as $g(x)$ and the corresponding subsequent layers as $f(\cdot)$. In this work we focus on the last layer. Thus, treat the last layer as $g(\cdot)$, i.e., we focus on models 
$
h(x)\triangleq sign(f(g(x))) = sign(w_h \cdot g(x) - \theta_h),   
$
where $w_h$ is a unit norm.
Thus, we can infer that $\nabla f(g(x))=w_h$ and $S(x) = w_h\cdot v.$ 
Our next theorem proves that $\textrm{TCAV$_{con}$}$ is similar to class-conditioned measure.

\begin{restatable}{theorem}{CompareTCAV}
\label{clm:compare_tcav}
Fix a concept $c(x)=g(x)\cdot v$ for some $v$ and $g$, and a predictor $h(x)=sign(w_h \cdot g(x) - \theta_h)$.
For any $\epsilon, \delta\in(0,1)$, if there are at least $2\ln(1/\delta)/\epsilon^2$ concept-labeled examples to calculate TCAV', and $\theta_h\geq 1-\epsilon^2/8$, then with probability at least $1-\delta$ 
  \[
  \lvert\mathbb{E}_{x\sim p}[c(x)| h(x)] - \textrm{TCAV$_{con}$} \rvert < \epsilon.
  \]
\end{restatable}

The theorem teaches us that TCAV measures the necessity of a concept for a prediction, which differs from the approach in the work of \citet{yeh2020completeness} that measures the sufficiency of the concept.


\paragraph{Efficient implementation.}
Both TCAV and its follow-up work \citep{bai2022concept} achieve their objective by learning the concept. TCAV learns the vector $v$ and the follow-up work learns the concept to find its gradient. In contrast to these techniques, our approach does not require learning the concept. As a result, calculating our measures is more efficient than TCAV and its follow-up work   \citep{kim2018interpretability,bai2022concept}.

\section{Experiments}\label{sec:experiments}
In this section, we provide several experiments to demonstrate the validity and applicability of our approach. In \Cref{sec:exp_validate_semantics} we validate that our measures capture the semantic relationship between the predictor and the concept. In \Cref{sec:exp_applications} we explore different applications of our method. 
Finally, in \Cref{sec:automatic_concept_labeling} we show that concepts can be automatically labeled by an image captioner.

\subsection{Necessary and Sufficient Explanations}\label{sec:exp_validate_semantics}

As a proof of concept, this section validates the semantic meaning of the measures through two scenarios. In the first scenario $h(x)$ predicts the fine-grained Felidae classes (e.g., Persian cat or cheetah). When $c(x)$ is the ``Felidae'' concept, the concept is necessary, and we expect the class-conditioned  measure to be high. When $c(x)$ is the irrelevant ``wolf'' concept we expect the measure to be low.
In the second scenario, when $h(x)$ predicts ``Felidae'' and $c(x)$ is the fine-grained Felidae concepts, the concept is sufficient for the predictor, and thus we expect the concept-conditioned measure to be high. If $h(x)$ predicts ``wolf'', the concepts are irrelevant and we expect the measure to be low.

Empirically, we show in Figure~\ref{fig:three graphs} that a pretrained ResNet-18 \citep{he2016deep} on ImageNet \citep{deng2009imagenet} exhibits this expected behavior. Specifically, among the $1000$ classes in ImageNet, we consider the classes under the ``Felidae'' hierarchy, these include class ids ranging from $281$ to $293$. When predicting the more coarse class (i.e. when we try to predict whether the image contains a Felidae or not, rather than recognizing the fine-grained type of Felidae), for simplicity we set $h(x)=1$ whenever the ResNet outputs any index in $\{281,\ldots, 293\}$.

\begin{figure}
     \centering
     \begin{subfigure}[b]{0.48\textwidth}
         \centering
         \includegraphics[width=\textwidth]{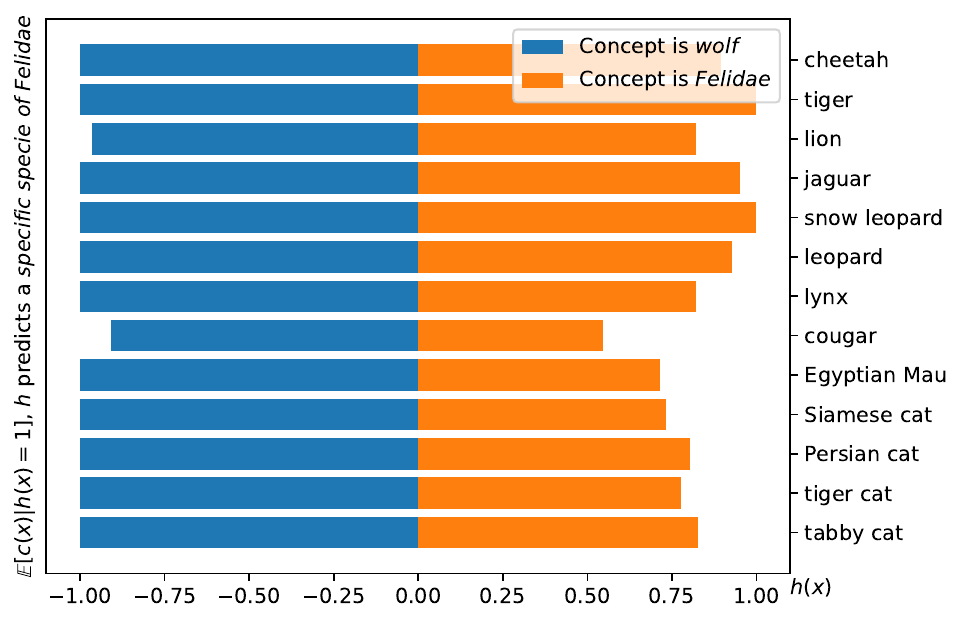}
     \end{subfigure}
     \hfill
     \begin{subfigure}[b]{0.48\textwidth}
         \centering
         \includegraphics[width=\textwidth]{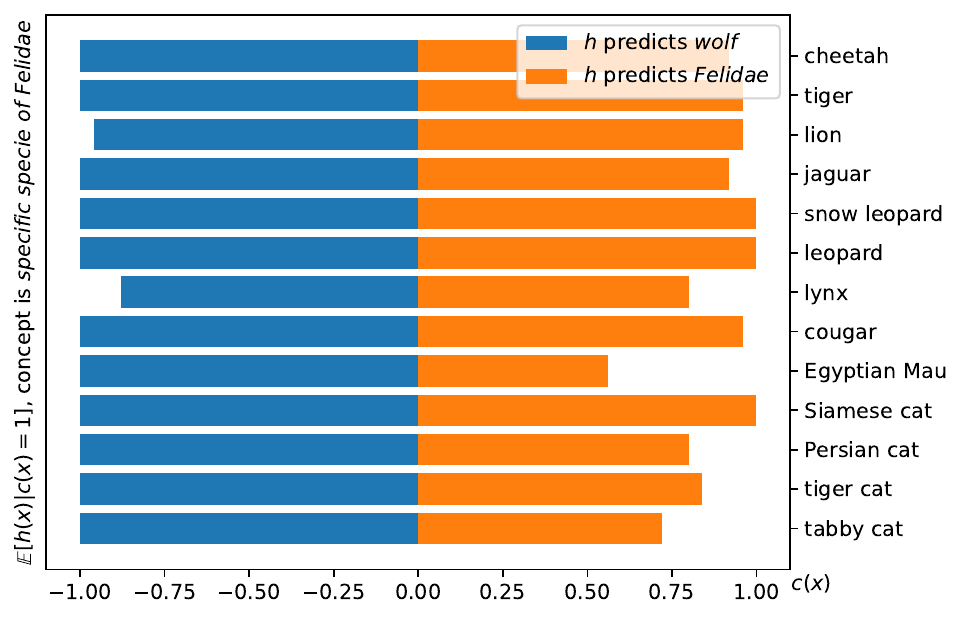}
     \end{subfigure}
    \caption{Necessary and sufficient concepts under our proposed measures. \textbf{Left}: $\mathbb{E}[c(x)|h(x)=1]$, where $h(x)$ predicts the fine-grained Felidae classes (e.g. Persian cat or cheetah), $c(x)$ is the ``Felidae'' or ``wolf'' concept. \textbf{Right}: $\mathbb{E}[h(x)|c(x)=1]$ where $h(x)=1$ means the image is a wolf or Felidae, $c(x)$ is the fine-grained Felidae concepts.}
    \label{fig:three graphs}
\end{figure}

\subsection{Applications}\label{sec:exp_applications}
In this section, we explore three different experiments demonstrating the applicability of the newly devised measures.
We conduct experiments on the a-Pascal dataset \citep{farhadi2009describing}. This dataset contains 6340 training images and 6355 testing images, distributed unevenly over 20 classes. Each image is associated with a concept vector in $\{0, 1\}^{64}$, with concepts like ``wing'' and ``shiny''. The true concepts are labeled by human experts, and in \Cref{sec:automatic_concept_labeling} we show that current image captioners can achieve reliable performance in labeling concepts as well. Unless otherwise specified, we use the human-labeled concepts in the experiments.


\subsubsection{Logistic regression versus random forest explanations}
A key advantage of our approach is its ability to operate with any model, allowing us to apply our measures to both logistic regression (LR) and random forest (RF) models. This sets our approach apart from previous research that primarily focused on neural networks.
In this section, we utilize these measures to compare LR and RF models.

We trained the LR classifier with a $\ell_2$ regularizer. The RF classifier has $100$ trees and each has a maximum depth of $15$. Both models are trained using \citet{scikit-learn}, hyperparameters are set to default unless stated otherwise. Both models achieve $87\%$ accuracy. We plot the model explanation with $\mathbb{E}[h(x)|c(x)=1]$ and $\mathbb{E}[c(x)|h(x)=1]$, see Figure~\ref{fig:lr_vs_rf}. As a comparison, we also draw $\mathbb{E}[c(x)|y=1]$ and $\mathbb{E}[y|c(x)=1]$, labeled as ``ground truth'' in the figures.

Our approach provides valuable insights into the dissimilarities of the models, despite LR and RF achieving almost identical accuracy.
Specifically, when predicting the cat class, RF weights more on the concepts ``tail'', ``torso'', and ``leg'', while LR puts more weight on ``eye''. When predicting the chair class, RF is more sensitive to the concept ``leather'' compared to LR, which is not aligned with the true data distribution.
These insights can be valuable in various scenarios, such as identifying spurious correlations and providing guidance during model selection. For instance, if a concept exhibits a high measure but lacks relevance, it can indicate the presence of spurious correlations. When faced with the choice between several models of equal quality, opting for the one that utilizes a greater number of relevant concepts can be a wise choice.

\begin{figure}
     \centering
     \begin{subfigure}[b]{0.49\textwidth}
         \centering
         \includegraphics[width=\textwidth]{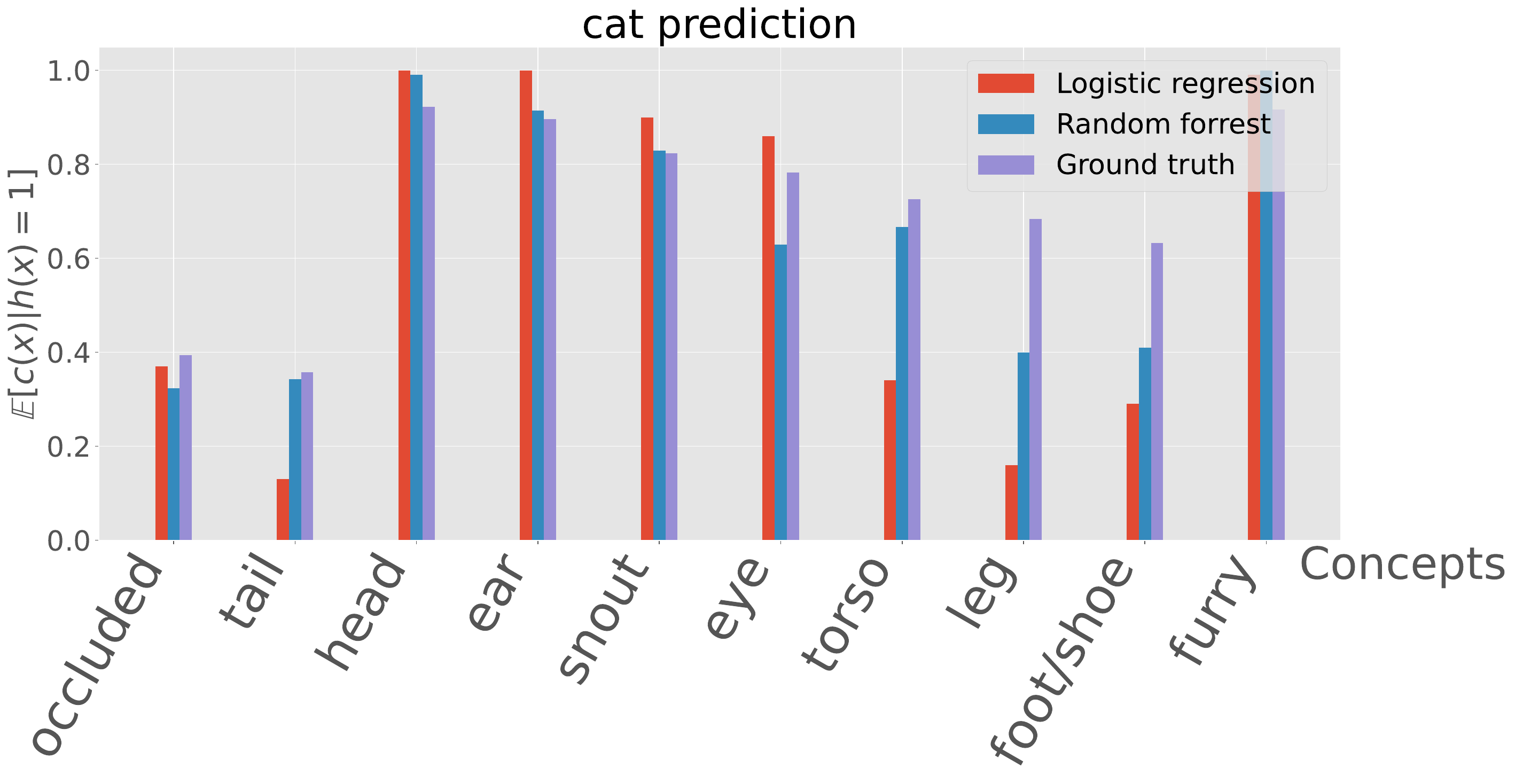}
     \end{subfigure}
     \hfill
     \begin{subfigure}[b]{0.49\textwidth}
         \centering
         \includegraphics[width=\textwidth]{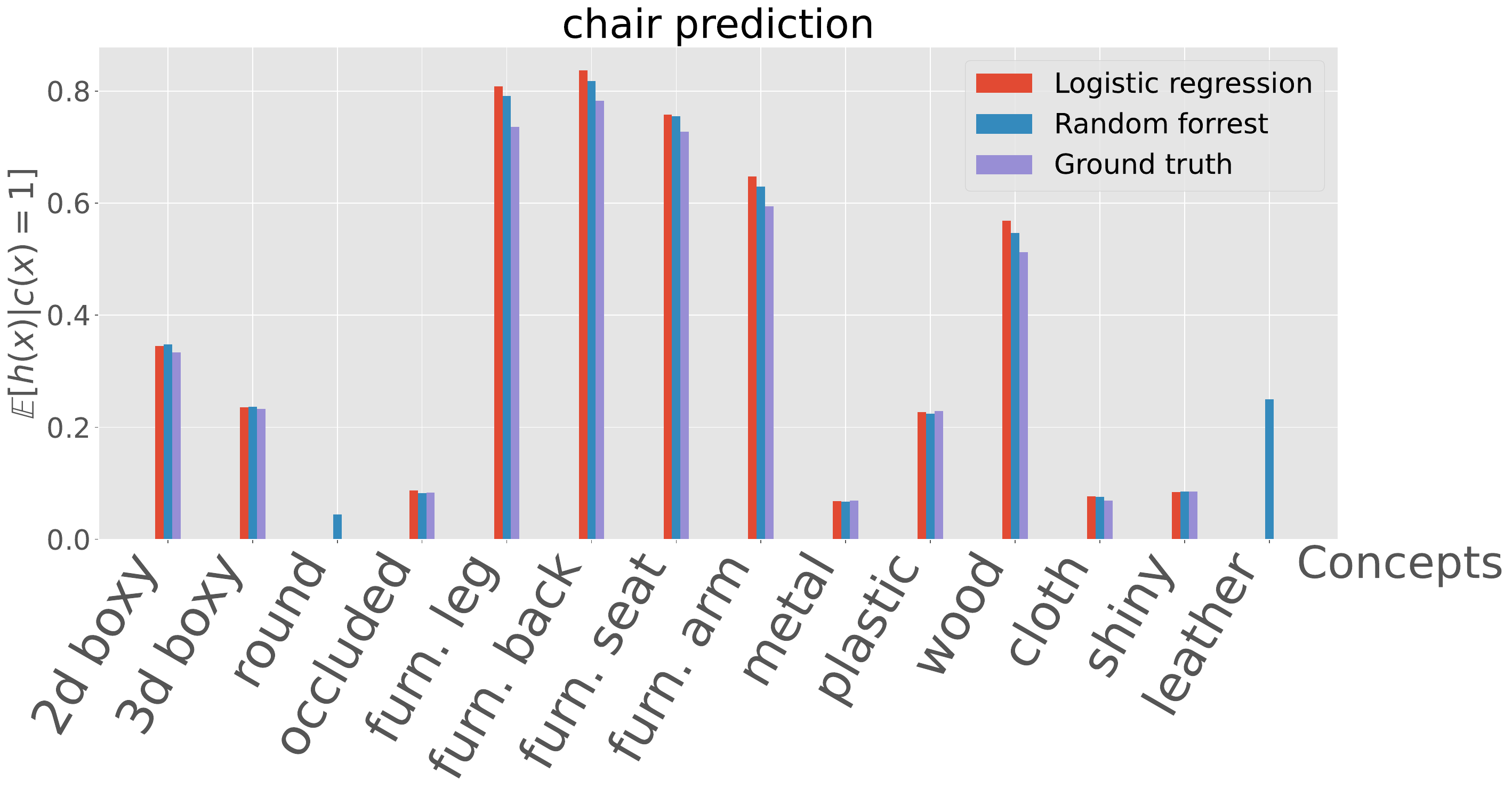}
     \end{subfigure}

    \caption{Logistic regression versus random forest. In both figures, the $x$-axis represents the concepts with positive measures, and the $y$-axis corresponds to a specific measure at interest. \textbf{Left}: $\mathbb{E}[c(x)|h(x)=1]$, where $h(x)=1$ predicts class ``cat''. \textbf{Right}: $\mathbb{E}[h(x)|c(x)=1]$, where $h(x)=1$ predicts class ``chair''.}
    \label{fig:lr_vs_rf}
\end{figure}

\subsubsection{SGD versus AdamW}
Our proposed method also provides a convenient way to study the behavior of different optimizers on the semantic level. The conventional wisdom has been that SGD finds a minimum that generalizes better while ADAM finds a local minimum faster \citep{zhou2020towards}, we are able to demonstrate what semantic concepts in the images are picked up by each optimizer, which can potentially lead to an explanation to the discrepancies in the optimizers' behaviors in addition to generalization ability.

We use ResNet-18 pretrained on ImageNet, and finetune with the following optimizers (unspecified hyperparameters are set to PyTorch defaults):
\begin{enumerate*}[(a)]
	\item SGD, last linear layer learning rate $1e-3$, other layer learning rate $1e-9$, momentum $0.9$, weight decay $5e-4$.
	\item AdamW, last linear layer learning rate $1e-3$, other layer learning rate $1e-9$.
\end{enumerate*}
\
Both models are trained $100$ epochs, the SGD model achieves a test accuracy of $58.46\%$ and the AdamW model achieves $60.26\%$.
Based on Figure~\ref{fig:sgd_vs_adam}, AdamW exhibits a tendency to overemphasize the motorbike class by giving excessive importance to irrelevant features such as "leg," while assigning insufficient weight to relevant features like "engine." This is in contrast to both SGD and the ground truth. This observation suggests that SGD is more suitable optimizer for this task.


\begin{figure*}
     \centering
     \begin{subfigure}[b]{\textwidth}
         \centering
         \includegraphics[width=\textwidth]{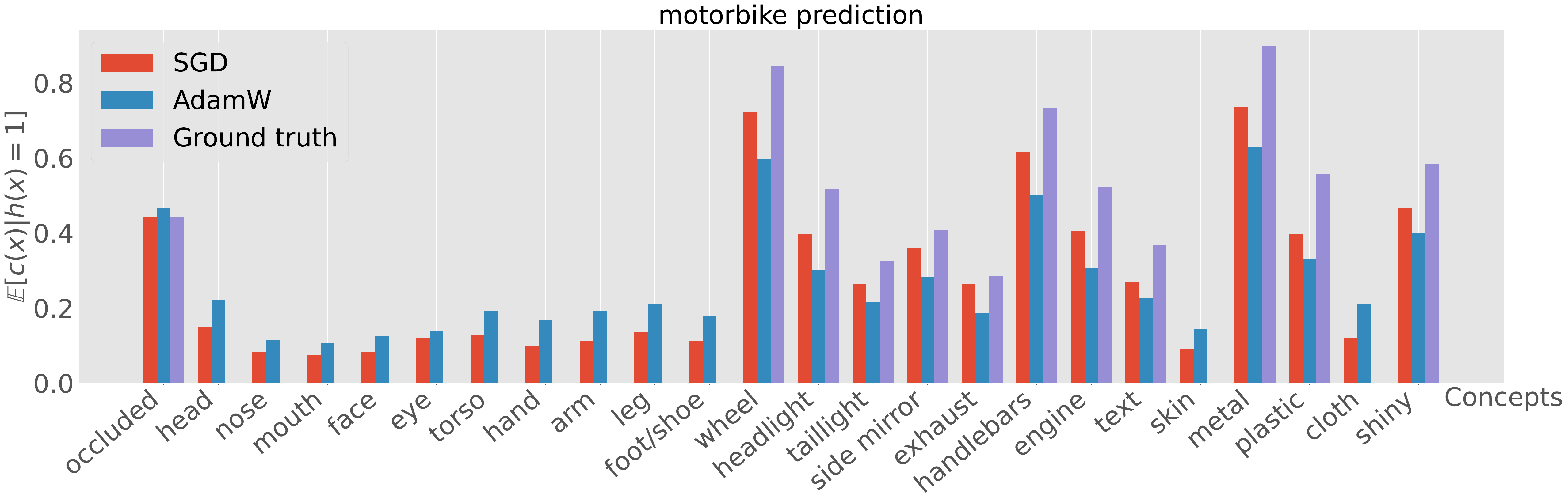}
     \end{subfigure}
     \hfill

    \caption{SGD versus AdamW with $\mathbb{E}[c(x)|h(x)=1]$, where $h(x)$ predicts ``motorbike'' class.}
    \label{fig:sgd_vs_adam}
\end{figure*}

\subsubsection{Prompt Editing}
This section demonstrates how our new approach can enhance model performance through prompt editing.
The recent development of vision-language models like CLIP \citep{radford2021learning} enables exceptional zero-shot classification performance via prompt engineering. These models have an image encoder $f(x): \mathbb{R}^d\to\mathbb{R}^h$ and a text encoder $g(t):\mathbb{R}^p\to\mathbb{R}^h$. These encoders are trained contrastively to align the matching $f(x)$ and $g(t)$, where the text $t$ is the true caption of the image $x$. During inference time, given a set of  $\mathcal{Z}$ classes, for each $z\in\mathcal{Z}$ we can construct a class prompt ``\texttt{a photo of} \{$t_z$\}'', and for an input image $x$, we predict:
$
	\hat z = \argmax_{z\in\mathcal Z} 	 f(x) \cdot g(t_z) .
$
We use the CLIP model with ViT-B16 \citep{dosovitskiy2020image} vision backbone on the a-Pascal dataset in this experiment, achieving a zero-shot test accuracy of $78.2\%$ and an F1-score of $0.6796$. However, we identify $10$ class prompts that mistakenly excessive emphasize on $13$ irrelevant concepts. For example, in Figure~\ref{fig:prompt_edit}, when conditioning on the original CLIP prompts to predict ``bottle'', there are many predicted images containing irrelevant concepts like ``head'', ``ear'', and ``torso''.

To counter the false emphasize, we edit the original class prompts of these $10$ classes by subtracting the irrelevant concepts. In particular, considering an irrelevant concept $w$ for class $z$, we use $g(t_z)-\lambda\cdot g(w)$ as the new prompt embedding for classifying class $z$, where $\lambda<1$ is a scalar that prevents us from over-subtracting. Heuristically we pick $\lambda=0.1$ based on trial-and-error using the training data, however, one can certainly learn $\lambda$ in a few-shot fashion. See more details in the appendix.  The edited prompts improve the F1-score to $0.7042$. Moreover, we can clearly see that the edited prompts align more with human intuition. \Cref{fig:prompt_edit} 
shows that irrelevant concepts have a smaller impact on edited prompts' predictions than on CLIP's predictions.

\begin{figure*}
     \centering
     \begin{subfigure}[b]{\textwidth}
         \centering
         \includegraphics[width=1.0\textwidth]{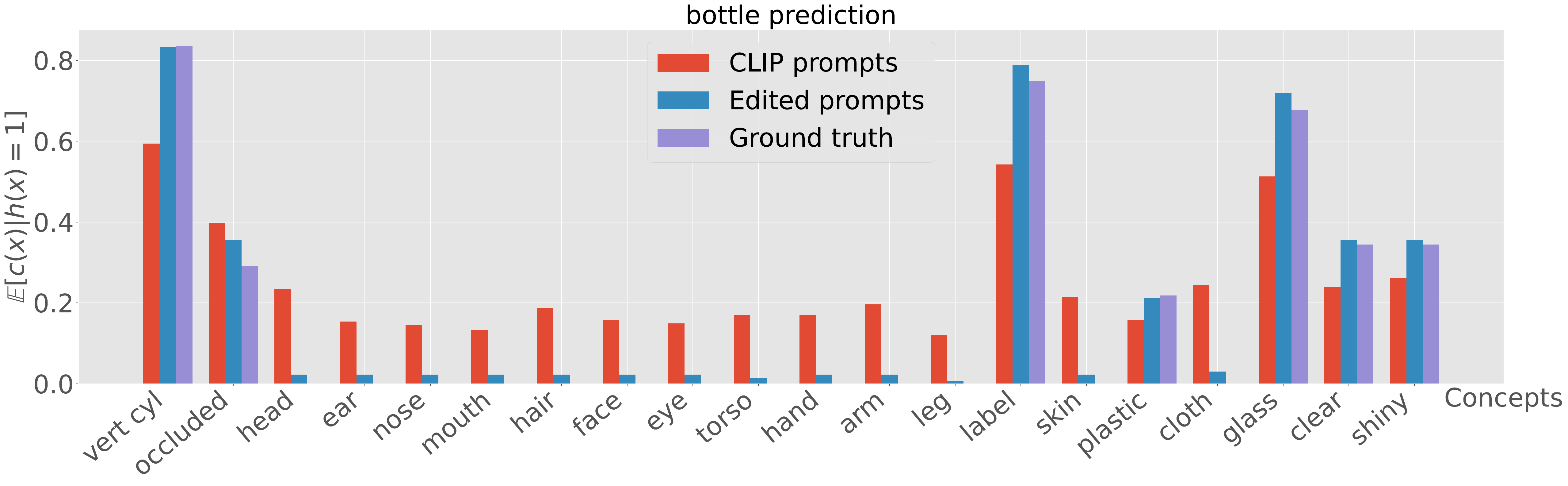}
     \end{subfigure}

    \caption{Prompt editing with $\mathbb{E}[c(x)|h(x)=1]$ where $h(x)=1$ predicts class ``bottle''.}
    \label{fig:prompt_edit}
\end{figure*}

\subsection{Automatic Concept Labeling with Image Captioner}\label{sec:automatic_concept_labeling}
Here we show that the state-of-the-art image captioners can reliably label the concepts in images. We use BLIP-2 \citep{li2023blip} with the pretrained FlanT5XL backbone \citep{chung2022scaling} with the prompt ``\texttt{Question: Does the \{classname\} in the picture have \{concept\}? Short Answer:}''. The maximum generation length is set to $1$. This configuration almost always only returns either ``yes'' or ``no''. To reduce the variance in answers, we ask ChatGPT for $10$ synonyms to the above-mentioned prompt. In total, we gather $11$ questions for each of the $64$ concepts. The final decisions are made based on thresholding the number of ``yes'' answers. Following \citet{bai2022concept}, we focus on the accuracy and recall metrics for concept retrieval. Recall@k means if the number of ``yes'' exceeds $k$, we label the concept as positive and compute the recall correspondingly. The same applies to acc@k. We try several different values of the thresholds and present the result in Table~\ref{table:acc_recall}.  


\begin{table}[!h]
\centering	
\caption{Concept labeling accuracy and recall at different thresholds, the model is FlanT5XL.}
\begin{tabular}{ c|c|c|c|c|c } 
\toprule
Acc@1 & Acc@3 & Acc@5 & Recall@1 & Recall@3 & Recall@5 \\
\midrule
 0.5120 & 0.7366 & 0.8071 &  0.9483 & 0.8520& 0.7672\\
\bottomrule
\end{tabular}
\label{table:acc_recall}
\end{table}

\section{Conclusion}\label{sec:conclusions}


In the paper, we proposed a novel axiomatic approach for concept explanations and developed an efficient algorithm to estimate it. Moreover, we investigated the connection between our approach and earlier techniques for concept explanation.  Subsequently, we harnessed the power of our newly developed algorithm across a range of scenarios, showcasing its many advantages:  it provides guidance to model selection, aids in choosing the optimizer for training neural networks, and, through prompt editing, enhances models' performance.


There are several directions for further research. Firstly, we suggest investigating concept learning and discovery to expand our understanding of how concepts are represented and identified. Given that there are potentially infinite concepts, efficiently learning a concise set of relevant concepts is essential for ensuring interpretability without human intervention.
Secondly, it would be interesting to explore the semantic meaning of different similarity functions to better understand how they influence concept explanations. 
Finally, we propose examining the use of concept explanation for transfer learning, where the method may help improve generalization performance and reduce overfitting.
%

\section{Acknowledgement}
Zhili Feng is supported by funding from the Bosch Center for Artificial Intelligence.

\bibliographystyle{plainnat}
\bibliography{bib}

\begin{thebibliography}{30}
\providecommand{\natexlab}[1]{#1}
\providecommand{\url}[1]{\texttt{#1}}
\expandafter\ifx\csname urlstyle\endcsname\relax
  \providecommand{\doi}[1]{doi: #1}\else
  \providecommand{\doi}{doi: \begingroup \urlstyle{rm}\Url}\fi

\bibitem[Bai et~al.(2022)Bai, Yeh, Ravikumar, Lin, and Hsieh]{bai2022concept}
Andrew Bai, Chih-Kuan Yeh, Pradeep Ravikumar, Neil~YC Lin, and Cho-Jui Hsieh.
\newblock Concept gradient: Concept-based interpretation without linear
  assumption.
\newblock \emph{arXiv preprint arXiv:2208.14966}, 2022.

\bibitem[Bontempelli et~al.(2022)Bontempelli, Teso, Tentori, Giunchiglia, and
  Passerini]{bontempelli2022concept}
Andrea Bontempelli, Stefano Teso, Katya Tentori, Fausto Giunchiglia, and Andrea
  Passerini.
\newblock Concept-level debugging of part-prototype networks.
\newblock \emph{arXiv preprint arXiv:2205.15769}, 2022.

\bibitem[Chen et~al.(2019)Chen, Li, Tao, Barnett, Rudin, and Su]{chen2019looks}
Chaofan Chen, Oscar Li, Daniel Tao, Alina Barnett, Cynthia Rudin, and
  Jonathan~K Su.
\newblock This looks like that: deep learning for interpretable image
  recognition.
\newblock \emph{Advances in neural information processing systems}, 32, 2019.

\bibitem[Chung et~al.(2022)Chung, Hou, Longpre, Zoph, Tay, Fedus, Li, Wang,
  Dehghani, Brahma, et~al.]{chung2022scaling}
Hyung~Won Chung, Le~Hou, Shayne Longpre, Barret Zoph, Yi~Tay, William Fedus,
  Eric Li, Xuezhi Wang, Mostafa Dehghani, Siddhartha Brahma, et~al.
\newblock Scaling instruction-finetuned language models.
\newblock \emph{arXiv preprint arXiv:2210.11416}, 2022.

\bibitem[Csisz{\'a}r(2008)]{csiszar2008axiomatic}
Imre Csisz{\'a}r.
\newblock Axiomatic characterizations of information measures.
\newblock \emph{Entropy}, 10\penalty0 (3):\penalty0 261--273, 2008.

\bibitem[Deng et~al.(2009)Deng, Dong, Socher, Li, Li, and
  Fei-Fei]{deng2009imagenet}
Jia Deng, Wei Dong, Richard Socher, Li-Jia Li, Kai Li, and Li~Fei-Fei.
\newblock Imagenet: A large-scale hierarchical image database.
\newblock In \emph{2009 IEEE conference on computer vision and pattern
  recognition}, pages 248--255. Ieee, 2009.

\bibitem[Dosovitskiy et~al.(2020)Dosovitskiy, Beyer, Kolesnikov, Weissenborn,
  Zhai, Unterthiner, Dehghani, Minderer, Heigold, Gelly,
  et~al.]{dosovitskiy2020image}
Alexey Dosovitskiy, Lucas Beyer, Alexander Kolesnikov, Dirk Weissenborn,
  Xiaohua Zhai, Thomas Unterthiner, Mostafa Dehghani, Matthias Minderer, Georg
  Heigold, Sylvain Gelly, et~al.
\newblock An image is worth 16x16 words: Transformers for image recognition at
  scale.
\newblock \emph{arXiv preprint arXiv:2010.11929}, 2020.

\bibitem[Espinosa~Zarlenga et~al.(2022)Espinosa~Zarlenga, Barbiero, Ciravegna,
  Marra, Giannini, Diligenti, Shams, Precioso, Melacci, Weller,
  et~al.]{espinosa2022concept}
Mateo Espinosa~Zarlenga, Pietro Barbiero, Gabriele Ciravegna, Giuseppe Marra,
  Francesco Giannini, Michelangelo Diligenti, Zohreh Shams, Frederic Precioso,
  Stefano Melacci, Adrian Weller, et~al.
\newblock Concept embedding models: Beyond the accuracy-explainability
  trade-off.
\newblock \emph{Advances in Neural Information Processing Systems},
  35:\penalty0 21400--21413, 2022.

\bibitem[Farhadi et~al.(2009)Farhadi, Endres, Hoiem, and
  Forsyth]{farhadi2009describing}
Ali Farhadi, Ian Endres, Derek Hoiem, and David Forsyth.
\newblock Describing objects by their attributes.
\newblock In \emph{2009 IEEE conference on computer vision and pattern
  recognition}, pages 1778--1785. IEEE, 2009.

\bibitem[Ghorbani et~al.(2019)Ghorbani, Wexler, Zou, and
  Kim]{ghorbani2019towards}
Amirata Ghorbani, James Wexler, James~Y Zou, and Been Kim.
\newblock Towards automatic concept-based explanations.
\newblock \emph{Advances in Neural Information Processing Systems}, 32, 2019.

\bibitem[Han et~al.(2022)Han, Srinivas, and Lakkaraju]{han2022explanation}
Tessa Han, Suraj Srinivas, and Himabindu Lakkaraju.
\newblock Which explanation should {I} choose? a function approximation
  perspective to characterizing post hoc explanations.
\newblock \emph{arXiv preprint arXiv:2206.01254}, 2022.

\bibitem[He et~al.(2016)He, Zhang, Ren, and Sun]{he2016deep}
Kaiming He, Xiangyu Zhang, Shaoqing Ren, and Jian Sun.
\newblock Deep residual learning for image recognition.
\newblock In \emph{Proceedings of the IEEE conference on computer vision and
  pattern recognition}, pages 770--778, 2016.

\bibitem[Karimi et~al.(2020)Karimi, Barthe, Balle, and Valera]{karimi2020model}
Amir-Hossein Karimi, Gilles Barthe, Borja Balle, and Isabel Valera.
\newblock Model-agnostic counterfactual explanations for consequential
  decisions.
\newblock In \emph{International Conference on Artificial Intelligence and
  Statistics}, pages 895--905. PMLR, 2020.

\bibitem[Kim et~al.(2018)Kim, Wattenberg, Gilmer, Cai, Wexler, Viegas, and
  Sayres]{kim2018interpretability}
Been Kim, Martin Wattenberg, Justin Gilmer, Carrie Cai, James Wexler, Fernanda
  Viegas, and Rory Sayres.
\newblock Interpretability beyond feature attribution: Quantitative testing
  with concept activation vectors ({TCAV}).
\newblock In \emph{International conference on machine learning}, pages
  2668--2677. PMLR, 2018.

\bibitem[Koh and Liang(2017)]{koh2017understanding}
Pang~Wei Koh and Percy Liang.
\newblock Understanding black-box predictions via influence functions.
\newblock In \emph{International conference on machine learning}, pages
  1885--1894. PMLR, 2017.

\bibitem[Koh et~al.(2020)Koh, Nguyen, Tang, Mussmann, Pierson, Kim, and
  Liang]{koh2020concept}
Pang~Wei Koh, Thao Nguyen, Yew~Siang Tang, Stephen Mussmann, Emma Pierson, Been
  Kim, and Percy Liang.
\newblock Concept bottleneck models.
\newblock In \emph{International Conference on Machine Learning}, pages
  5338--5348. PMLR, 2020.

\bibitem[Koh et~al.(2019)Koh, Ang, Teo, and Liang]{koh2019accuracy}
Pang Wei~W Koh, Kai-Siang Ang, Hubert Teo, and Percy~S Liang.
\newblock On the accuracy of influence functions for measuring group effects.
\newblock \emph{Advances in neural information processing systems}, 32, 2019.

\bibitem[Krishna et~al.(2022)Krishna, Han, Gu, Pombra, Jabbari, Wu, and
  Lakkaraju]{krishna2022disagreement}
Satyapriya Krishna, Tessa Han, Alex Gu, Javin Pombra, Shahin Jabbari, Steven
  Wu, and Himabindu Lakkaraju.
\newblock The disagreement problem in explainable machine learning: A
  practitioner's perspective.
\newblock \emph{arXiv preprint arXiv:2202.01602}, 2022.

\bibitem[Li et~al.(2023)Li, Li, Savarese, and Hoi]{li2023blip}
Junnan Li, Dongxu Li, Silvio Savarese, and Steven Hoi.
\newblock Blip-2: Bootstrapping language-image pre-training with frozen image
  encoders and large language models.
\newblock \emph{arXiv preprint arXiv:2301.12597}, 2023.

\bibitem[Lundberg and Lee(2017)]{lundberg2017unified}
Scott~M Lundberg and Su-In Lee.
\newblock A unified approach to interpreting model predictions.
\newblock \emph{Advances in neural information processing systems}, 30, 2017.

\bibitem[Pedregosa et~al.(2011)Pedregosa, Varoquaux, Gramfort, Michel, Thirion,
  Grisel, Blondel, Prettenhofer, Weiss, Dubourg, Vanderplas, Passos,
  Cournapeau, Brucher, Perrot, and Duchesnay]{scikit-learn}
F.~Pedregosa, G.~Varoquaux, A.~Gramfort, V.~Michel, B.~Thirion, O.~Grisel,
  M.~Blondel, P.~Prettenhofer, R.~Weiss, V.~Dubourg, J.~Vanderplas, A.~Passos,
  D.~Cournapeau, M.~Brucher, M.~Perrot, and E.~Duchesnay.
\newblock Scikit-learn: Machine learning in {P}ython.
\newblock \emph{Journal of Machine Learning Research}, 12:\penalty0 2825--2830,
  2011.

\bibitem[Radford et~al.(2021)Radford, Kim, Hallacy, Ramesh, Goh, Agarwal,
  Sastry, Askell, Mishkin, Clark, et~al.]{radford2021learning}
Alec Radford, Jong~Wook Kim, Chris Hallacy, Aditya Ramesh, Gabriel Goh,
  Sandhini Agarwal, Girish Sastry, Amanda Askell, Pamela Mishkin, Jack Clark,
  et~al.
\newblock Learning transferable visual models from natural language
  supervision.
\newblock In \emph{International conference on machine learning}, pages
  8748--8763. PMLR, 2021.

\bibitem[Ribeiro et~al.(2016)Ribeiro, Singh, and Guestrin]{ribeiro2016should}
Marco~Tulio Ribeiro, Sameer Singh, and Carlos Guestrin.
\newblock " why should {I} trust you?" explaining the predictions of any
  classifier.
\newblock In \emph{Proceedings of the 22nd ACM SIGKDD international conference
  on knowledge discovery and data mining}, pages 1135--1144, 2016.

\bibitem[Ribeiro et~al.(2018)Ribeiro, Singh, and Guestrin]{ribeiro2018anchors}
Marco~Tulio Ribeiro, Sameer Singh, and Carlos Guestrin.
\newblock Anchors: High-precision model-agnostic explanations.
\newblock In \emph{Proceedings of the AAAI conference on artificial
  intelligence}, volume~32, 2018.

\bibitem[Shalev-Shwartz and Ben-David(2014)]{shalev2014understanding}
Shai Shalev-Shwartz and Shai Ben-David.
\newblock \emph{Understanding machine learning: From theory to algorithms}.
\newblock Cambridge university press, 2014.

\bibitem[Sundararajan et~al.(2017)Sundararajan, Taly, and
  Yan]{sundararajan2017axiomatic}
Mukund Sundararajan, Ankur Taly, and Qiqi Yan.
\newblock Axiomatic attribution for deep networks.
\newblock In \emph{International conference on machine learning}, pages
  3319--3328. PMLR, 2017.

\bibitem[Wachter et~al.(2017)Wachter, Mittelstadt, and
  Russell]{wachter2017counterfactual}
Sandra Wachter, Brent Mittelstadt, and Chris Russell.
\newblock Counterfactual explanations without opening the black box: Automated
  decisions and the gdpr.
\newblock \emph{Harv. JL \& Tech.}, 31:\penalty0 841, 2017.

\bibitem[Yeh et~al.(2020)Yeh, Kim, Arik, Li, Pfister, and
  Ravikumar]{yeh2020completeness}
Chih-Kuan Yeh, Been Kim, Sercan Arik, Chun-Liang Li, Tomas Pfister, and Pradeep
  Ravikumar.
\newblock On completeness-aware concept-based explanations in deep neural
  networks.
\newblock \emph{Advances in Neural Information Processing Systems},
  33:\penalty0 20554--20565, 2020.

\bibitem[Yuksekgonul et~al.(2023)Yuksekgonul, Wang, and
  Zou]{yuksekgonul2022post}
Mert Yuksekgonul, Maggie Wang, and James Zou.
\newblock Post-hoc concept bottleneck models.
\newblock In \emph{International Conference on Learning Representations}, 2023.

\bibitem[Zhou et~al.(2020)Zhou, Feng, Ma, Xiong, Hoi, et~al.]{zhou2020towards}
Pan Zhou, Jiashi Feng, Chao Ma, Caiming Xiong, Steven Chu~Hong Hoi, et~al.
\newblock Towards theoretically understanding why sgd generalizes better than
  adam in deep learning.
\newblock \emph{Advances in Neural Information Processing Systems},
  33:\penalty0 21285--21296, 2020.

\end{thebibliography}

\appendix
\section{Proofs for Section~\ref{sec:axiomatic_approach}}


\AxiomTwoClm*
\begin{proof}
From the first axiom we know that $M$ can be written as 
\[ 
M(h, c, p) = \sum_x m(h(x), c(x), p(x)),
\]
for some $m$. For any example $x_0$, the recursivity axiom tells us that 
\[
m(h(x_0),c(x_0), p_1+p_2) + \sum_{x\neq x_0} m(h(x),c(x), p(x)) \]
is equal to 
\[m(h(x_0),c(x_0), p_1) + m(h(x_0),c(x_0),p_2) + \sum_{x\neq x_0} m(h(x),c(x), p(x))
\]

So overall, when focusing on $m$ as a function of $p$,  the recursivity axiom implies that 
\begin{equation}\label{eq:additivity_of_m}
  m(\cdot, \cdot, p_1+p_2) = m(\cdot, \cdot, p_1)+m(\cdot, \cdot, p_2)  
\end{equation}
It is a well-known fact that the last equation implies that $m$ is linear with respect to the probability $p(x)$ over the rational numbers. For the sake of completeness, this fact is proved in the next claim. Thus, $m$ can be written as $$m(p(x), c(x), h(x))=p(x) \cdot s(c(x), h(x)).$$
\end{proof}

\begin{claim}
$m$ is linear with respect to $p$. 
\end{claim}

\begin{proof}
Recall that a linear function f(x) is a function that satisfies two properties: (1) additivity: $f(x + y) = f(x) + f(y)$ and
(2) $f(\alpha x) = \alpha f(x)$ for all $\alpha$. The first property is Equation~\ref{eq:additivity_of_m}. Thus, we need to prove that this equation also imply property (2), when $\alpha$ is rational. We will show it in three steps, first when $\alpha$ is integer, then when $1/\alpha$ is integer, and finally when $\alpha$ is rational, i.e., $\alpha=a/b$ when $a$ and $b$ are integers. 

For integer $a$, and probability $p$, it holds that \[m(ap) = m(p + p + \ldots +p) = am(p).\] For integer $1/a$, it holds that 
\[
m(p)=m(\frac1a \cdot a p )=\frac1a \cdot m(ap),
\]
which implies that, \[am(p) = m(ap).\]
For rational $a/b$ it holds that 
\[
m(a/b \cdot p ) = m(a \cdot 1/b \cdot p)  = am(1/b\cdot p) = \frac ab m(p).
\]

\end{proof}


\section{Proofs for Section~\ref{sec:comparison}}\label{apx:proof_comparison}

For a model $h$, the completeness score \cite{yeh2020completeness} is defined as 
\[
\frac{\sup_g \Pr_{x,y} [y= \argmax_{y'} h_{y'}(g(\mathcal{C}(x)) ] - a_r }{ \Pr_{x,y} [y = \argmax_{y'} f_{y'}(x)]  -a_r}, 
\]
where $\mathcal{C}(x)$ are the concepts extracted from the input instance $x$, $a_r$ is the accuracy of a random prediction to equate the lower bound
of completeness score to 0, and $f$ is the neural network to be explained. 
In our setting, the denominator is equal to 1. Thus, up to normalization, the completeness score is equal to 
\[
\sup_g \Pr_{x,y} [y= \argmax_{y'} h_{y'}(g(\mathcal{C}(x)) ].
\] In this paper we focus on measures of one concept, hence $C(x) = c(x)$. And recall that we denoted 
\begin{equation}
S_{CACBE}\triangleq\sup_g \Pr_{x,y}[y = \argmax_{y'} h_{y'}(g(c(x))],
\end{equation}

Now we are ready to prove the Theorems in Section~\ref{sec:comparison}.

\CompareAware*
\begin{proof}
Focus on the LHS of \Cref{clm:aware_equiv}. 
Partition all examples into four parts, depending on their prediction $h(x)$ and their concept $c(x)$. Let $g:\{-1, +1\}\to\mathcal{X}$, and recall that at the beginning of \Cref{sec:comparison} we set $y\triangleq h(x)$. Then, the LHS is equal to 
\begin{align*}
& \Pr_x(1=h(g(1))| c(x) = 1 \wedge h(x) =1]) \Pr(c(x) = 1 \wedge h(x) =1)\\
+ & \Pr_x(-1=h(g(1))| c(x) = 1 \wedge h(x) = -1]) \Pr(c(x) = 1 \wedge h(x) =-1)\\
+ & \Pr_x(1=h(g(-1))| c(x) = -1 \wedge h(x) =1]) \Pr(c(x) = -1 \wedge h(x) =1)\\
+ & \Pr_x(-1=h(g(-1))| c(x) = -1 \wedge h(x) = -1]) \Pr(c(x) = -1 \wedge h(x) =-1)
\end{align*}
To maximize the first two terms, $h(g(1))$ should be $1$ if $$\Pr(c(x) = 1 \wedge h(x) =1)>\Pr(c(x) = 1 \wedge h(x) =-1)$$ and otherwise $h(g(1)) = -1$. Similar argument applies to the last two terms. 
Hence, the LHS in Equation~\ref{clm:aware_equiv} is equal to 
\begin{equation}
  \begin{aligned}\label{eq:clm_proof_1}
 & \max(\Pr(c(x) = 1 \wedge h(x) =1), \Pr(c(x) = 1 \wedge h(x) = -1))\\ & +  \max(\Pr(c(x) = -1 \wedge h(x) =-1), \Pr(c(x) = -1 \wedge h(x) =-1)) 
\end{aligned}
\end{equation}

We can rewrite the above probabilities, for $y = -1, +1$ 
\begin{align*}
 &   \Pr(h(x) = +1 \wedge c(x) = y) = \mathbb{E}_x\left[\frac{h(x)+1}{2}\; |\;c(x)=y\right]\cdot\Pr(c(x)=y)
    \\
& \Pr(h(x) = -1 \wedge c(x) = y) = \mathbb{E}_x\left[\frac{-h(x)+1}{2}\;|\;c(x)=y\right]\cdot\Pr(c(x)=y)
\end{align*}

In each term of
 Equation~\ref{eq:clm_proof_1}, taking the max translates to taking an absolute value, as can be seen by the following equations 
 \begin{align*}
 &\max \Big(\mathbb{E}_x\left[\frac{h(x)+1}{2}\; |\;c(x)=y\right]\cdot\Pr(c(x)=y), \mathbb{E}_x\left[\frac{-h(x)+1}{2}\;|\;c(x)=y\right]\cdot\Pr(c(x)=y)\Big)\\ 
 & =  \frac12 \Big(\max(\mathbb{E}_x\left[h(x)\; |\;c(x)=y\right], -\mathbb{E}_x\left[h(x)\;|\;c(x)=y\right])+1\Big) \cdot\Pr(c(x)=y)\\
  & =  \frac12(\rvert\mathbb{E}_{x}[h(x)|c(x)=y]\lvert+1)\cdot \Pr(c(x)=y),
 \end{align*}
which is equal to a single term in the sum on the RHS of Equation~\ref{clm:aware_equiv}, thus proving the claim. 
\end{proof}

\CompareTCAV*
\begin{proof}
Before stating the full proof, here are the 
main ideas. The proof begins with estimating the class-conditioned measure using the formula $\frac{\sum_{x\in X_k} c(x)}{|X_k|}$, which is closely related to the TCAV definition. Then we need to show that $c(x)$ is similar to $S(x)$ when $x$ is labeled as $1$. Such examples satisfy that $x$ and $w_h$ are close to each other. This intuitively holds because $c(x)=g(x)\cdot v$ measures the closeness of $x$ and $v$, while $S(x)=w_h\cdot v$ measures the closeness of $w_h$ and $v$. Since $x$ and $w_h$ are close, so are $S(x)$ and $c(x)$, which proves the theorem. To prove the theorem formally we apply Cauchy–Schwarz inequality, triangle inequality, and Hoeffding’s inequality.

Now we countinue with the formal proof. 
Since $v$ and all the embeddings are normalized to have a unit norm, the Cauchy-Schwarz inequality implies that the value of $c(x)$ falls within the range of $[-1,1]$. Therefore, we can use Hoeffding's inequality as shown in Fact~\ref{fact:hoeffding}, to prove that the average concept in the class is a reliable estimator,
    \begin{equation}\label{eq:implication_of_hoeffding}
        \left\lvert  \mathbb{E}_{x\sim p}[c(x)| h(x)] - \frac{\sum_{x\in X_k} c(x)}{|X_k|} \right\rvert < \frac\epsilon2. 
    \end{equation}
Therefore, it is sufficient to provide a bound for the following term
\[
\left\lvert \frac{\sum_{x\in X_k} c(x)}{|X_k|} -\textrm{TCAV$_{con}$} \right\rvert = \left\lvert  \frac{\sum_{x\in X_k} g(x)\cdot v}{|X_k|} - \frac{\sum_{x\in X_k} w_h\cdot v}{|X_k|} \right\rvert. 
\]
By applying the triangle inequality, it is satisfactory to establish a bound for each of the inner terms in the sum in order to prove our claim. Specifically, we need to bound the following term to support our assertion.
\[
|g(x)\cdot v - w_h\cdot v| = |(g(x)-w_h)\cdot v|
\]
By Cauchy–Schwarz inequality, the last term is bounded by 
\[
\lVert g(x)- w_h\rVert\cdot \lVert v \rVert = \lVert g(x)- w_h\rVert = \sqrt{\lVert g(x) \rVert^2 + \lVert w_h\rVert^2 -2g(x)\cdot w_h} = \sqrt{2(1-g(x)\cdot w_h)},
\]
where the first equality follows from the fact that $v$ is a unit norm and the last equality follows from the fact that $g(x)$ and $w_h$ are unit norms. Knowing that $x$ belongs to $X_k$, we can deduce that $w_h\cdot g(x) - \theta_h$ is greater than zero. Implying that $g(x)\cdot w_h \geq 1 - \epsilon^2/8,$ which is equivalent to $1-g(x)\cdot w_h \leq \epsilon^2/8.$ This proves that 
\[
|g(x)\cdot v - w_h\cdot v| \leq \frac\epsilon2.
\]
The assertion's statement follows when combining both the last inequality and Inequality~\ref{eq:implication_of_hoeffding}.
\end{proof}

\begin{fact}[Hoeffding's inequality]\label{fact:hoeffding}
For $X_1,\ldots, X_n\in[-1,1]$ independent random variables
\[
\Pr\left(
\left\lvert\frac1n \sum_i \mathbb{E}[X_i] - \frac{1}{n}\sum_i X_i\right\rvert > \epsilon\right) < \exp(-n\epsilon^2/2)
\]
\end{fact}

\section{Experiment Details}
\subsection{Prompt editing with few-shot data}
As mentioned in \cref{sec:experiments}, we can leverage few-shot data when editing prompts with the help of the new measures.
We pick $10$ classes of which the measure $\mathbb{E}[c(x)|h(x)=1]$ does not align with the ground truth (which also aligns with human intuition). In particular, ``bicycle'',  ``bird'',  ``bottle'',  ``chair'',  ``diningtable'',  ``horse'',  ``motorbike'',  ``sofa'',  ``train'',  ``tvmonitor'' are the selected classes. In \Cref{fig:prompt_edit_all_classes} we see $\mathbb{E}[c(x)|h(x)=1]$ on all $20$ classes, and among them the above-mentioned $10$ classes are selected for prompt editing. For simplicity, we focus only on a subset of $13$ concepts that we want to modify (specifically here, we only want to subtract them), which include ``head'', ``ear'', ``nose'', ``mouth'', ``hair'', ``face'', ``eye'', ``torso'', ``hand'', ``arg'', ``leg'', ``foot/shoe'', ``skin''. 

To edit the prompt, for any $z$ in the above classes, we first get the normalized\footnote{CLIP always normalizes its embedding to the unit norm.} text embedding $\mathcal{I}_z$ of its prompt ``\texttt{a photo of \{z\}}''. Then for each concept $c\in\mathcal{C}$ that we want to subtract, we get its normalized text embedding $\mathcal{I}_c$. Finally, we use $\mathcal{I}_z-\lambda \cdot \frac{1}{|\mathcal{C}|}\sum_{c\in\mathcal{C}}\mathcal{I}_c$ as the prompt embedding, where $\lambda=0.1$. We see that after editing, the measure $\mathbb{E}[c(x)|h(x)=1]$ aligns with the ground truth more, see \Cref{fig:prompt_edit_all_classes}.

Instead of heuristically setting $\lambda=0.1$ for every class, we can instead use $\mathcal{I}_z- \cdot \lambda_z\frac{1}{|\mathcal{C}|}\sum_{c\in\mathcal{C}}\mathcal{I}_c$, and learn $\lambda_z$. We use $k=1,2,4,8,16$ shot per class for training. We train using AdamW with a weight decay of $1$ for $1000$ epochs. They achieve a superior F1-score compared to the original CLIP prediction (see \Cref{table:fs_scalar_f1}), and also a superior accuracy (see \Cref{table:fs_scalar_acc}) for large enough $k$.

\begin{table}[tb]
\centering	
\caption{F1-score (standard deviation) of the original CLIP prediction and few-shot learned $\lambda_z$}
\resizebox{\textwidth}{!}{
\begin{tabular}{ c|c|c|c|c|c } 
\toprule
 Original & $k=1$ & $k=2$ & $k=4$ & $k=8$ & $k=16$ \\
\midrule
 0.6796 & 0.7001 (0.0073) & 0.6903 (0.0089) & 0.6996 (0.0037) & 0.6991 (0.0044) & 0.7003 (0.0027) \\ 
\bottomrule
\end{tabular}
}
\label{table:fs_scalar_f1}
\end{table}

\begin{table}[tb]
\centering	
\caption{Accuracy (standard deviation) of the original CLIP prediction and few-shot learned $\lambda_z$}
\resizebox{\textwidth}{!}{
\begin{tabular}{ c|c|c|c|c|c } 
\toprule
 Original & $k=1$ & $k=2$ & $k=4$ & $k=8$ & $k=16$ \\
\midrule
0.7820 & 0.7637 (0.0138) & 0.7795 (0.0047) & 0.7762 (0.0047) & 0.7848 (0.0024) & 0.7844 (0.0025) \\ 
\bottomrule
\end{tabular}
}
\label{table:fs_scalar_acc}
\end{table}

\begin{figure}
     \centering
     \begin{subfigure}[b]{\textwidth}
         \centering
         \includegraphics[width=1.0\textwidth]{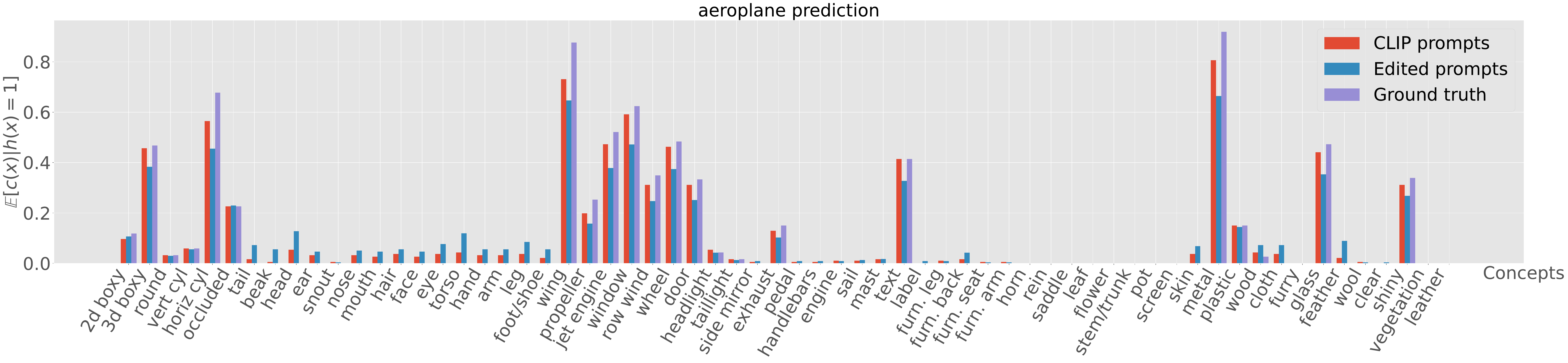}
     \end{subfigure}
    \begin{subfigure}[b]{\textwidth}
         \centering
         \includegraphics[width=1.0\textwidth]{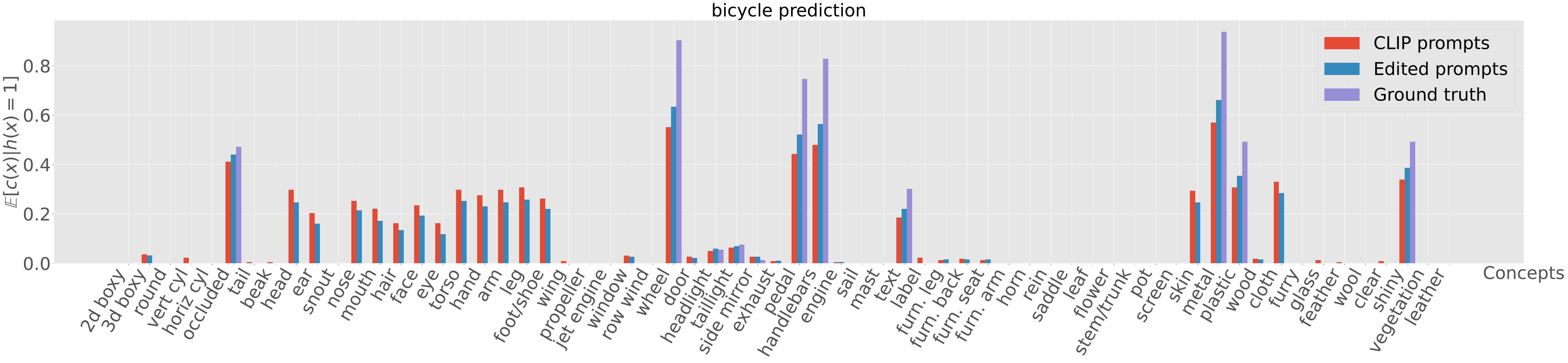}
     \end{subfigure}
     \begin{subfigure}[b]{\textwidth}
         \centering
         \includegraphics[width=1.0\textwidth]{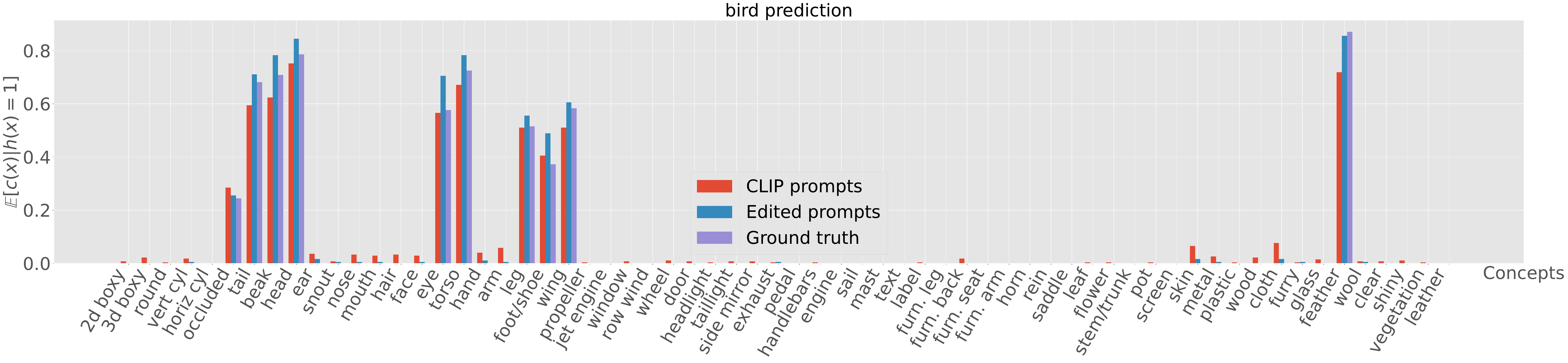}
     \end{subfigure}
     \begin{subfigure}[b]{\textwidth}
         \centering
         \includegraphics[width=1.0\textwidth]{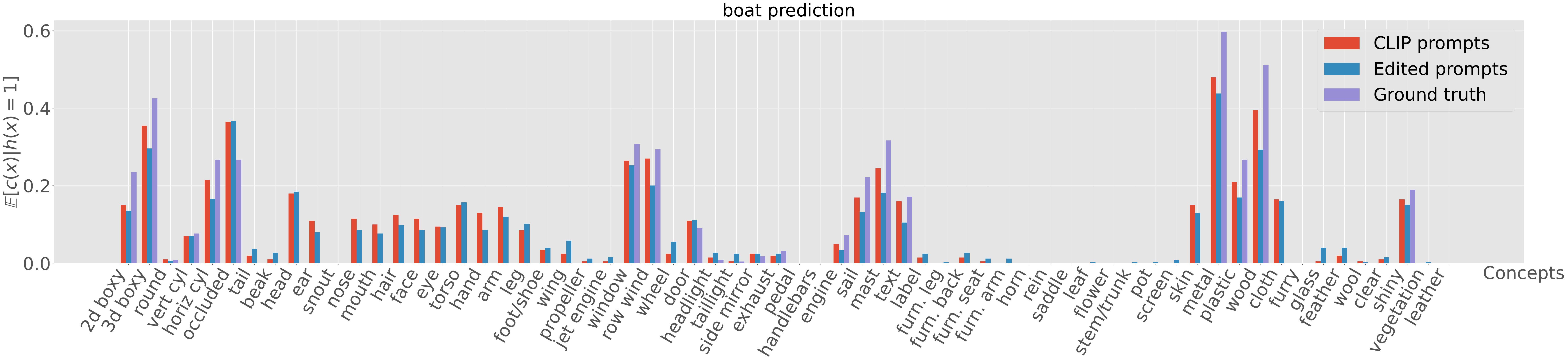}
     \end{subfigure}
     \begin{subfigure}[b]{\textwidth}
         \centering
         \includegraphics[width=1.0\textwidth]{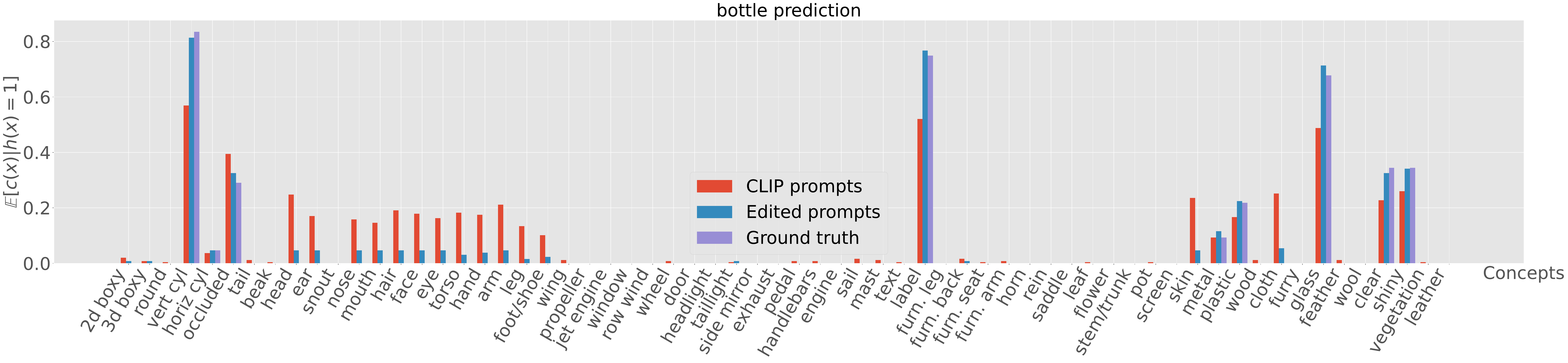}
     \end{subfigure}
 \end{figure}
 
 \begin{figure}
     \begin{subfigure}[b]{\textwidth}
         \centering
         \includegraphics[width=1.0\textwidth]{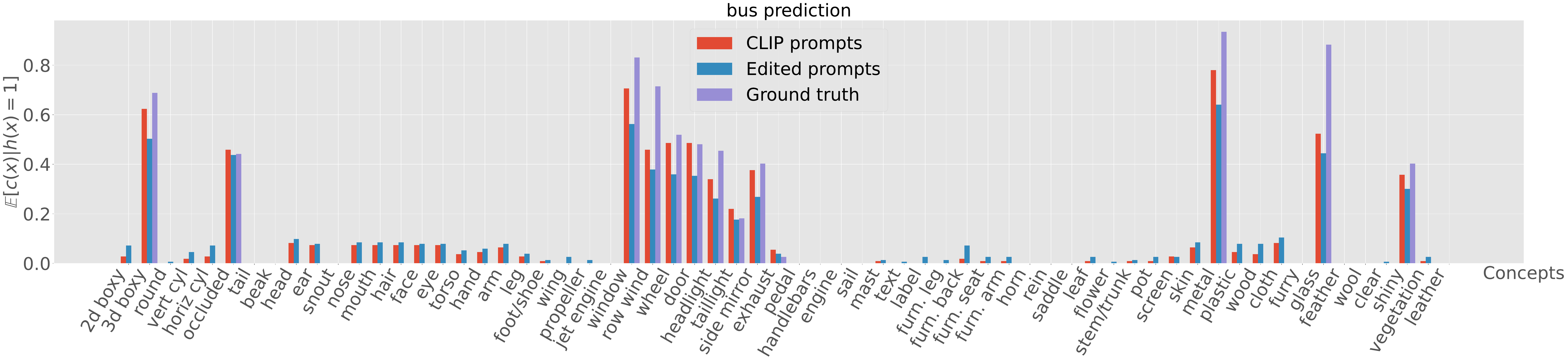}
     \end{subfigure}
     \begin{subfigure}[b]{\textwidth}
         \centering
         \includegraphics[width=1.0\textwidth]{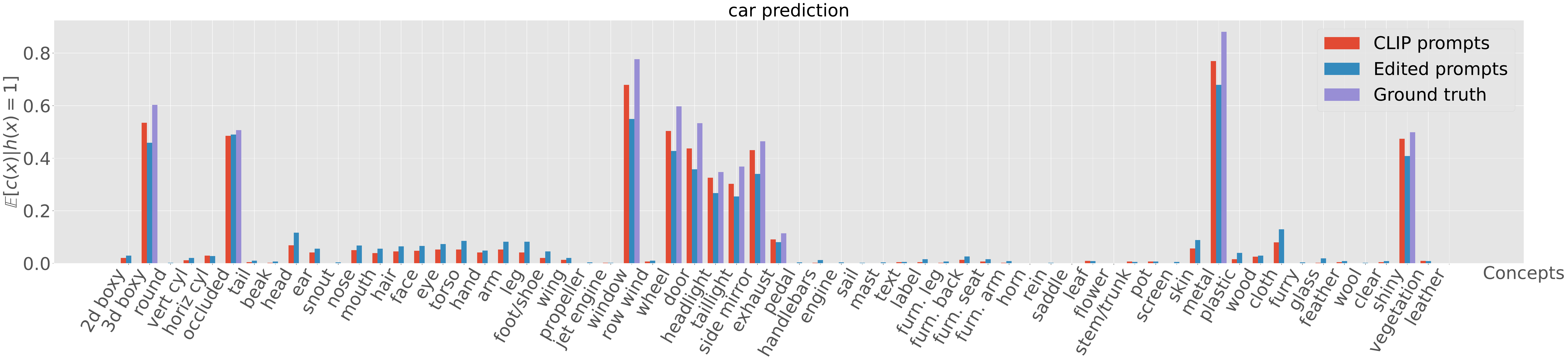}
     \end{subfigure}
     \begin{subfigure}[b]{\textwidth}
         \centering
         \includegraphics[width=1.0\textwidth]{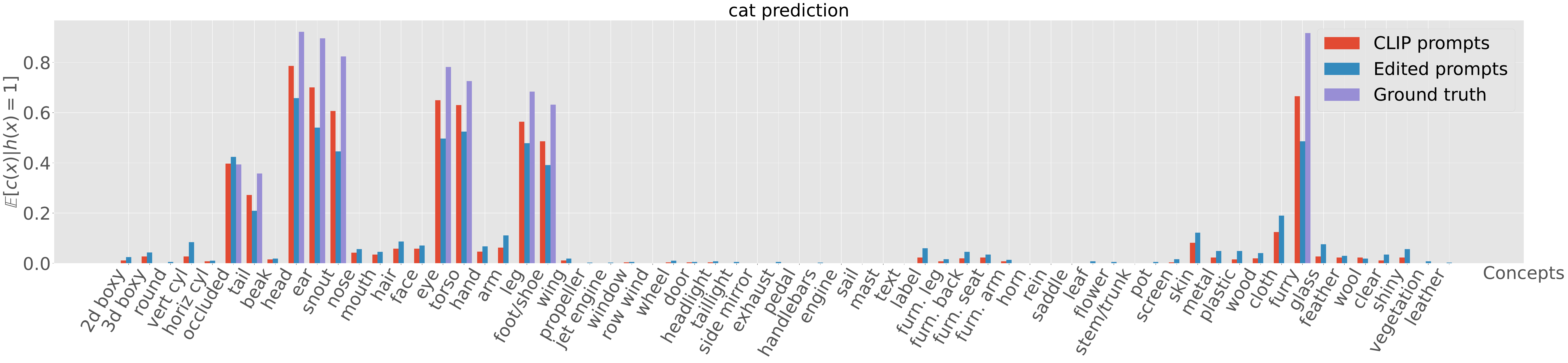}
     \end{subfigure}
     \begin{subfigure}[b]{\textwidth}
         \centering
         \includegraphics[width=1.0\textwidth]{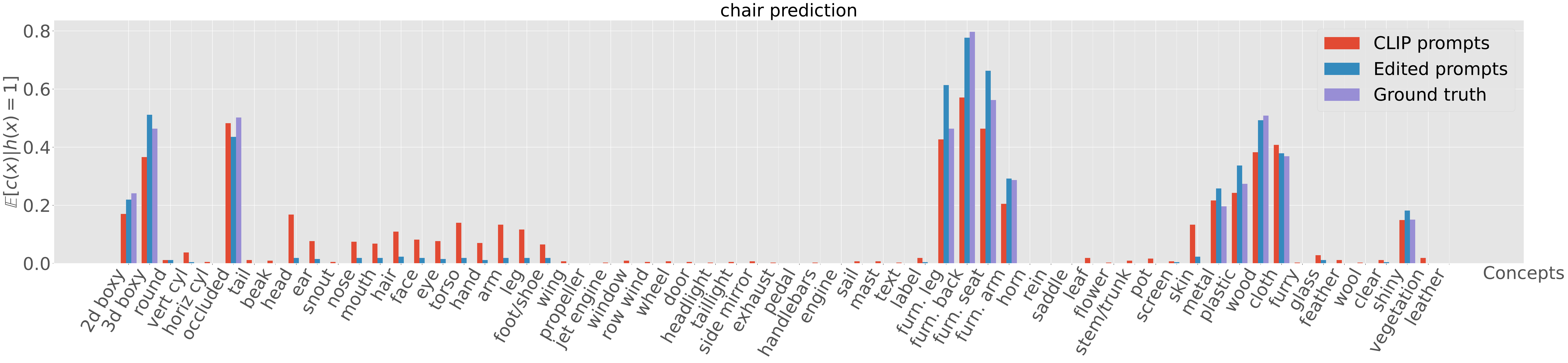}
     \end{subfigure}
     \begin{subfigure}[b]{\textwidth}
         \centering
         \includegraphics[width=1.0\textwidth]{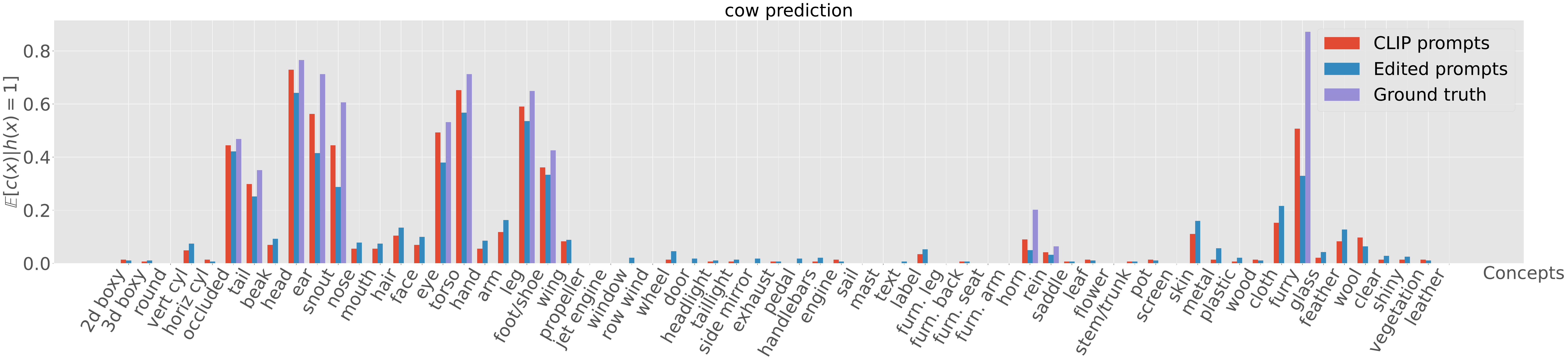}
     \end{subfigure}
\end{figure}

 \begin{figure}
     \begin{subfigure}[b]{\textwidth}
         \centering
         \includegraphics[width=1.0\textwidth]{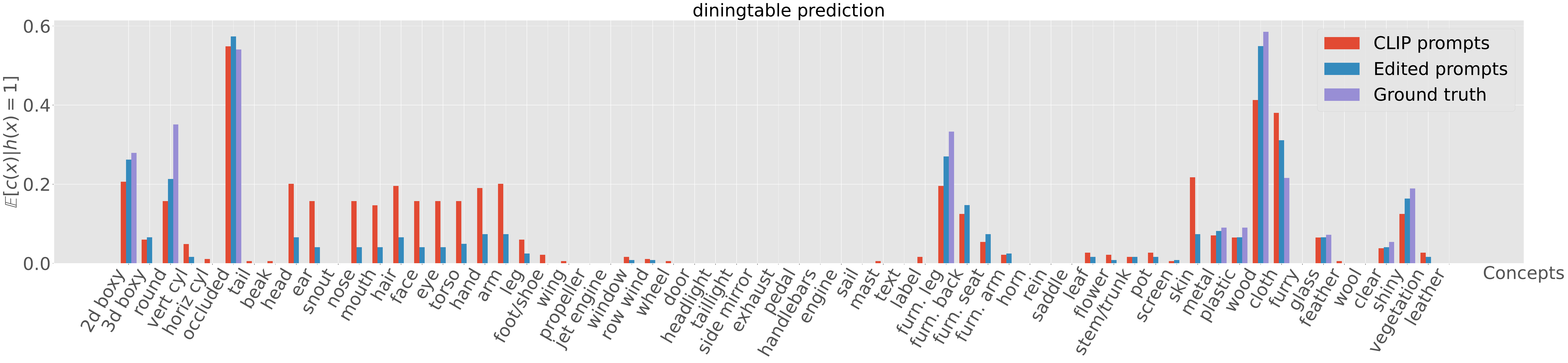}
     \end{subfigure}
     \begin{subfigure}[b]{\textwidth}
         \centering
         \includegraphics[width=1.0\textwidth]{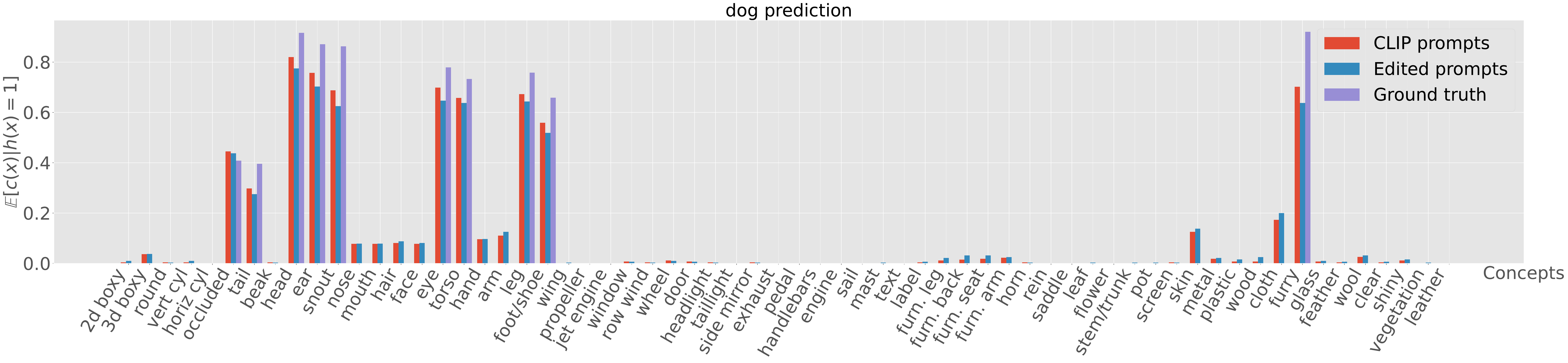}
     \end{subfigure}
     \begin{subfigure}[b]{\textwidth}
         \centering
         \includegraphics[width=1.0\textwidth]{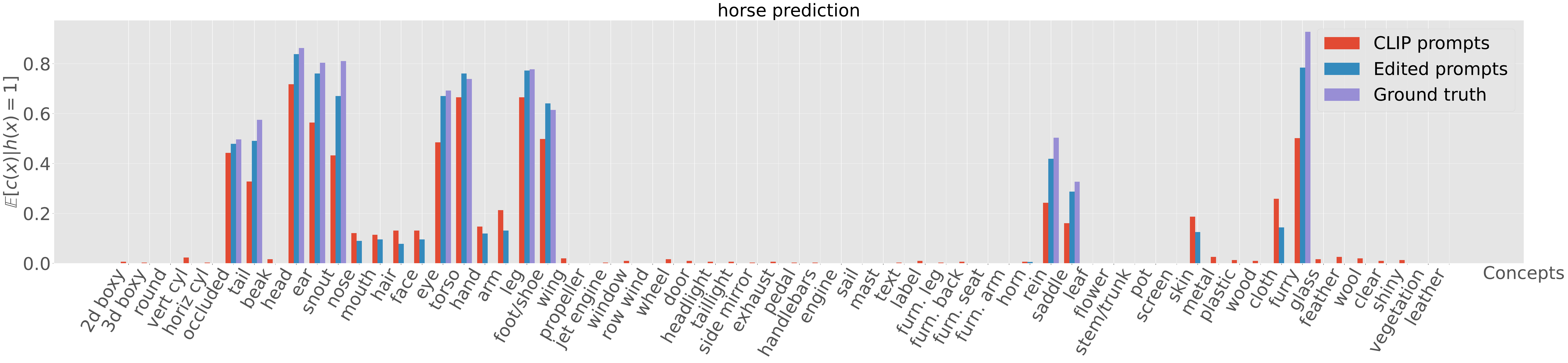}
     \end{subfigure}
     \begin{subfigure}[b]{\textwidth}
         \centering
         \includegraphics[width=1.0\textwidth]{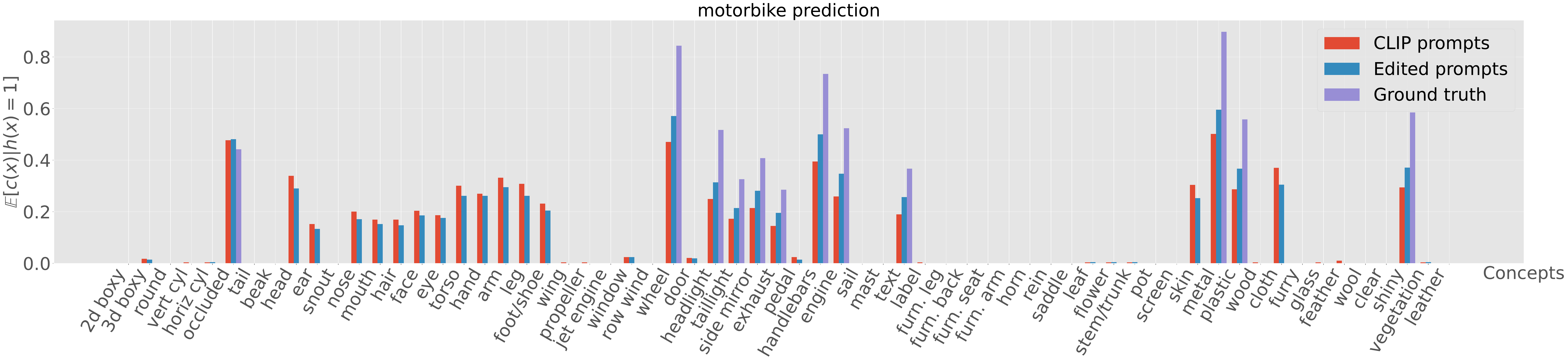}
     \end{subfigure}
     \begin{subfigure}[b]{\textwidth}
         \centering
         \includegraphics[width=1.0\textwidth]{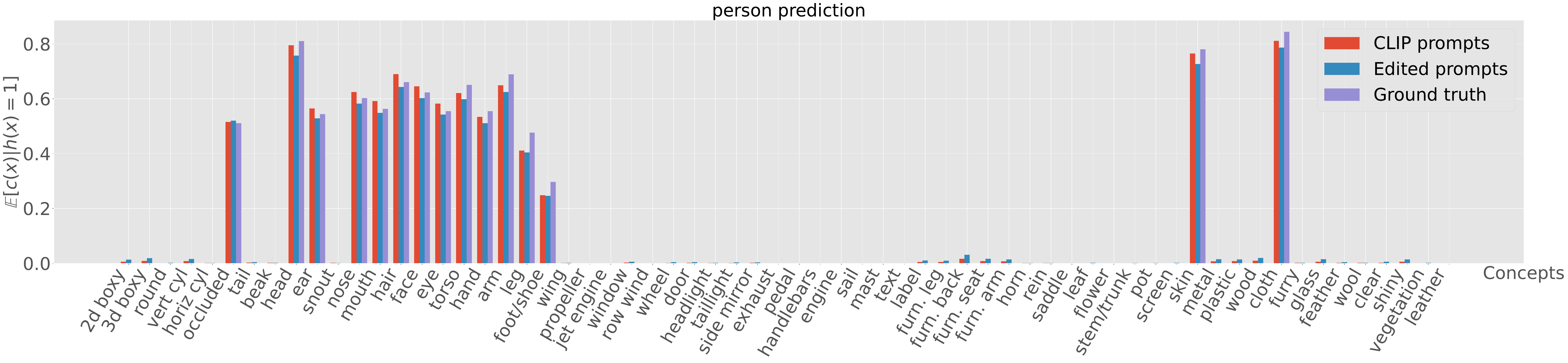}
     \end{subfigure}
\end{figure}

\begin{figure}
     \begin{subfigure}[b]{\textwidth}
         \centering
         \includegraphics[width=1.0\textwidth]{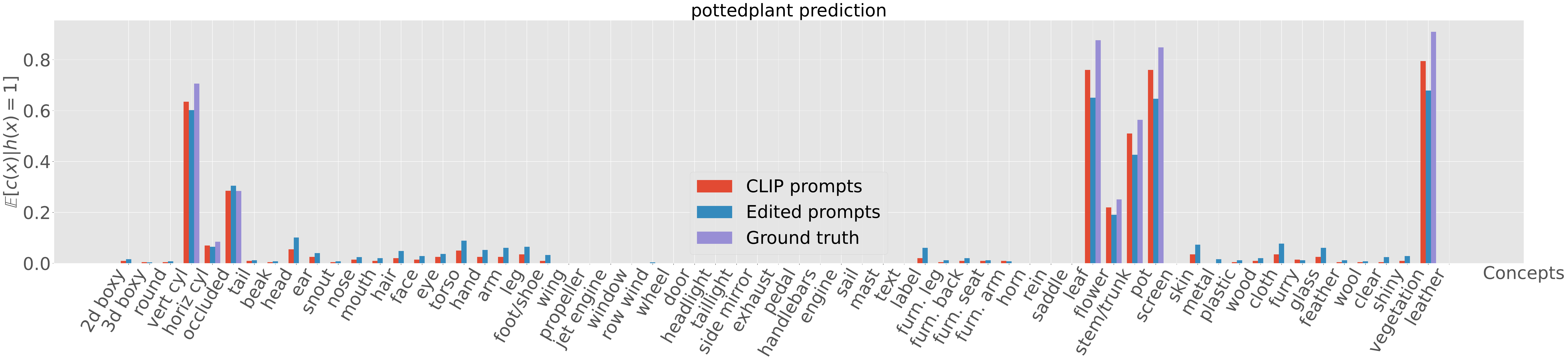}
     \end{subfigure}
     \begin{subfigure}[b]{\textwidth}
         \centering
         \includegraphics[width=1.0\textwidth]{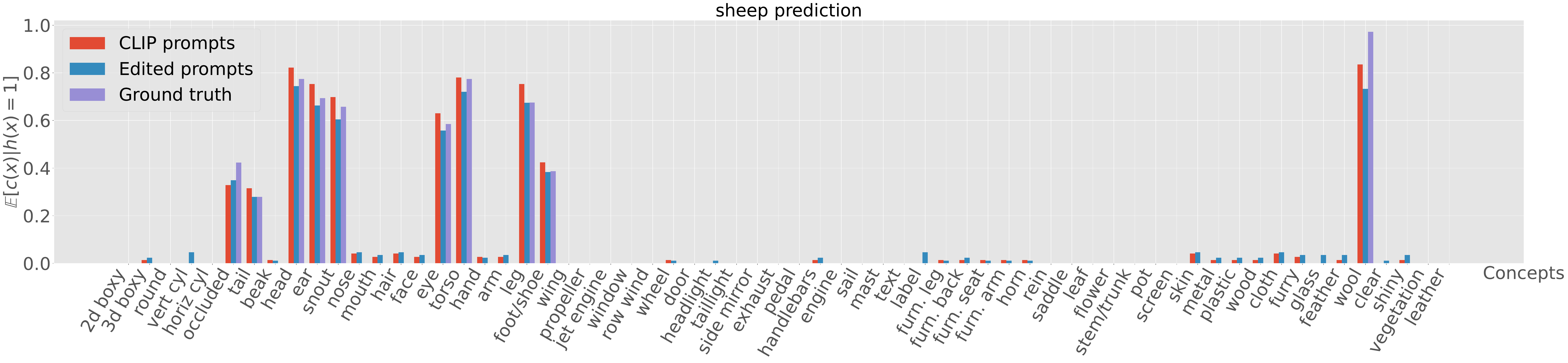}
     \end{subfigure}
     \begin{subfigure}[b]{\textwidth}
         \centering
         \includegraphics[width=1.0\textwidth]{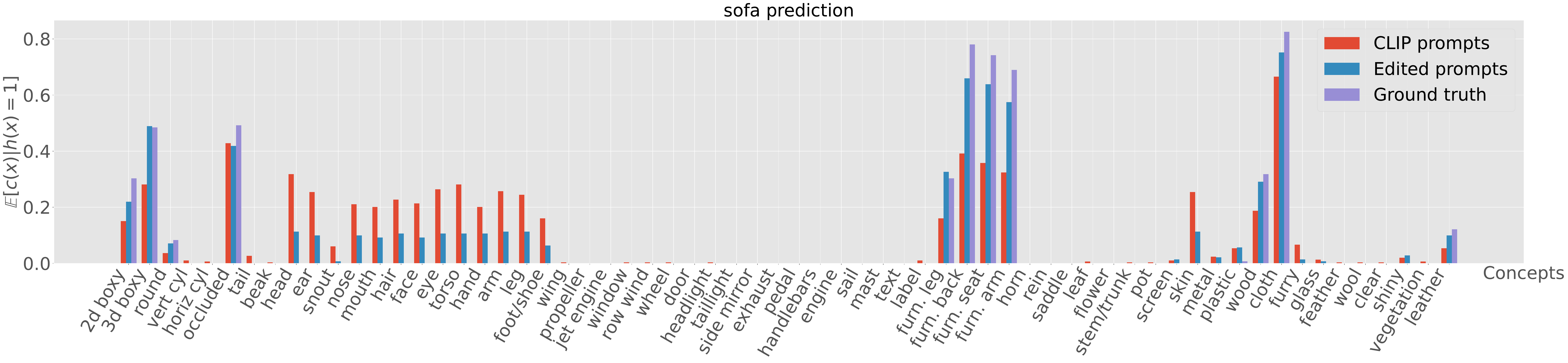}
     \end{subfigure}
     \begin{subfigure}[b]{\textwidth}
         \centering
         \includegraphics[width=1.0\textwidth]{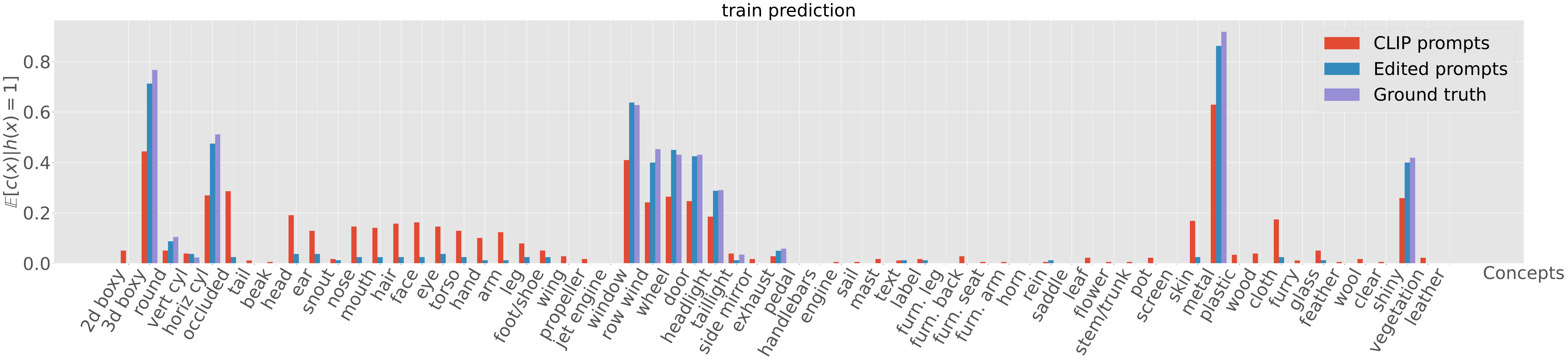}
     \end{subfigure}
     \begin{subfigure}[b]{\textwidth}
         \centering
         \includegraphics[width=1.0\textwidth]{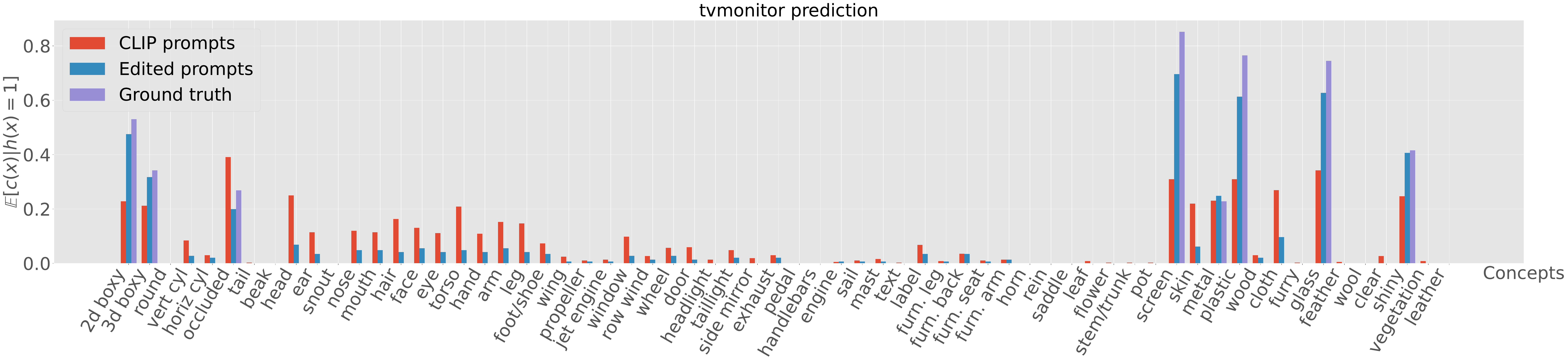}
     \end{subfigure}
    \caption{$\mathbb{E}[c(x)|h(x)=1]$ with original CLIP prompts, edited prompts, and ground truth.}
    \label{fig:prompt_edit_all_classes}
\end{figure}

\end{document}